\documentclass[11pt,a4paper]{article}

\usepackage[margin=1.2in]{geometry}

\usepackage{graphicx,url}
\usepackage[utf8]{inputenc}  
\usepackage{amsfonts,amsmath,amssymb,amsthm} 
\usepackage{graphicx}
\usepackage{hyperref}
\usepackage{xcolor}
\usepackage{booktabs}
\usepackage{authblk}

\newtheorem{theorem}{Theorem}

\newtheorem{lemma}{Lemma}

\theoremstyle{definition}
\newtheorem{definition}{Definition}
\newtheorem{example}{Example}
\newtheorem{remark}{Remark}

\newcommand{\ii}{\boldsymbol{i}}

\newcommand{\vetx}{\boldsymbol{x}}

\newcommand{\qeq}{\quad \text{and} \quad}

\newcommand{\basis}{\mathcal{E}=\{\boldsymbol{e}_0,\ldots,\boldsymbol{e}_{n-1}\}}

\title{Universal Approximation Theorem for Vector- and Hypercomplex-Valued Neural Networks\thanks{This work was supported in part by the National Council for Scientific and Technological Development (CNPq) under grant no 315820/2021-7, the S\~ao Paulo Research Foundation (FAPESP) under grant no 2022/01831-2, and the Coordena\c{c}\~ao  de Aperfei\c{c}oamento  de Pessoal de N\'ivel Superior - Brazil (CAPES) - Finance Code 001.}
}

\author[1]{Marcos Eduardo Valle\thanks{Email: valle@ime.unicamp.br, ORCID:0000-0003-4026-5110}}
\author[1,2]{Wington L. Vital\thanks{Email: wington.vital.BE@eldorado.org.br, ORCID:0000-0003-1634-4441}}
\author[1]{
Guilherme Vieira\thanks{Email: vieira.g@ime.unicamp.br, ORCID:0000-0003-3361-6154}}
\affil[1]{Universidade Estadual de Campinas (UNICAMP), Campinas -- Brazil.}
\affil[2]{Instituto de Pesquisa Eldorado, Campinas -- Brazil.}

\date{\today}

\begin{document}

\maketitle

\begin{abstract}
The universal approximation theorem states that a neural network with one hidden layer can approximate continuous functions on compact sets with any desired precision. This theorem supports using neural networks for various applications, including regression and classification tasks. Furthermore, it is valid for real-valued neural networks and some hypercomplex-valued neural networks such as complex-, quaternion-, tessarine-, and Clifford-valued neural networks. However, hypercomplex-valued neural networks are a type of vector-valued neural network defined on an algebra with additional algebraic or geometric properties. This paper extends the universal approximation theorem for a wide range of vector-valued neural networks, including hypercomplex-valued models as particular instances. Precisely, we introduce the concept of non-degenerate algebra and state the universal approximation theorem for neural networks defined on such algebras.

\end{abstract}

\section{Introduction}
Artificial neural networks (ANNs) are models initially inspired by biological neural networks' behavior. The origins of ANNs can be traced back to the pioneering works of McCulloch and Pitts \cite{mcculloch1943logical} and Rosenblatt \cite{rosenblatt1958perceptron}. Since their inception, ANNs have found wide-ranging applications in diverse fields such as computer vision, physics, control, pattern recognition, economics, and medicine. Their approximation capability partially supports the broad and successful applicabilities of ANNs. The approximation capability of ANNs is based on representation theorems that provide a theoretical foundation to approximate functions with high precision by properly tuning their parameters.

Several researchers delved into the approximation capability of fully connected neural networks, also known as multilayer perceptrons (MLPs), during the late 1980s and early 1990s. Notably, Cybenko, Funahashi, Chen et al., Hornik, Leshno, and their collaborators established conditions on the activation functions for the approximation capability of single hidden layer MLP networks with arbitrarily many hidden neurons. Specifically, Cybenko proved the approximation capability by assuming the activation function is both continuous and discriminatory in 1989 \cite{Cybenko1989ApproximationFunction}. 
In the same year, Funahashi independently reached the approximation capability by considering nonconstant, bounded, and monotone continuous activation function \cite{funahashi1989approximate}. In a symposium held at Michigan State University in 1990, Chen et al. proved that a bounded generalized sigmoid activation function, not necessarily continuous, yields the approximation capability \cite{chen1990constructive}.
Subsequently, Hornik showed that a function is discriminatory if it is bounded and nonconstant \cite{hornik1991approximation}. Finally, Leshno and collaborators concluded that single hidden layer MLP networks with locally bounded piecewise continuous activation functions have approximation capability if, and only if, the activation function is not a polynomial \cite{leshno1993multilayer}. The interested reader can ﬁnd a detailed and up-to-date discussion on the approximation capability of neural networks and related topics in the book \cite{Ismailov2021RidgeNetworks}.

In the 1990s, Arena et al. made a significant breakthrough by extending the universal approximation theorem for single hidden layer MLP networks based on complex numbers and quaternions \cite{ARENA1997335,arena1998neural}. Their theorems became essential in formulating universal approximation theorems for other hypercomplex-valued neural networks, that is, neural networks whose inputs, outputs, and parameters are hypercomplex numbers \cite{buchholz2000hyperbolic,Buchholz2001,Carniello2021UniversalNetworks,Voigtlaender2023TheNetworks}. Besides the theoretical developments, hypercomplex-valued neural networks have been successfully applied in various machine learning tasks \cite{hirose12,parcollet2020survey,comminiello2024demystifying}. For example, complex-valued neural networks have been effectively applied to communication channel prediction \cite{Ding2020OnlinePrediction} and classification of electrical disturbances in low-voltage networks \cite{IturrinoGarcia2023PowerLines}. Quaternion-valued neural networks have been used for the classification and analysis of synthetic aperture radar (SAR) images \cite{Matsumoto2022Full-LearningClassification,shang14}, sound localization and detection \cite{Brignone2022EfficientDomain}, and modeling human emotion expressions \cite{Guizzo2023LearningDomain}. Apart from the models based on complex numbers and quaternions, hypercomplex-valued neural networks have been used for robot manipulator control \cite{Takahashi2021ComparisonControl}, times series prediction and analysis \cite{vieira2022general,Navarro-Moreno2022ProperDomain}, computer-assisted diagnosis \cite{Vieira2022AcuteNetworks}, sound localization \cite{Grassucci2023DualRepresentation}, and color image classification \cite{Grassucci2022PHNNs:Convolutions}.

As pointed out in the previous paragraph, the first results on the approximation capability of single hidden layer hypercomplex-valued MLP networks date back to the works of Arena et al. \cite{ARENA1997335,arena1998neural}. In the early 2000s, Buchholz and Sommer successfully extended the universal approximation theorem for neural networks defined on Clifford algebras, including hyperbolic-valued neural networks \cite{buchholz2000hyperbolic,Buchholz2001}. More recently, Voigtlaender revised the universal approximation theorem for complex-valued neural networks by completely characterizing the class of activation functions for which a complex-valued neural network exhibits approximation capability \cite{Voigtlaender2023TheNetworks}. Moreover, Carniello et al. extended the universal approximation theorem for tessarine-valued neural networks \cite{Carniello2021UniversalNetworks}. In \cite{Vital2022ExtendingNetworks}, we investigated the approximation capability of hypercomplex-valued neural networks and enunciated (but did not provide the proof of) the universal approximation theorem for a broad class of hypercomplex-valued neural networks. 

This paper not only extends our conference paper \cite{Vital2022ExtendingNetworks} by including proofs for the universal approximation theorem for hypercomplex-valued neural networks, but it also considers a broader framework. Specifically, this paper addresses the approximation capability of vector-valued neural networks (V-nets), which are built using bilinear products \cite{Fan2020BackpropagationProducts,valle2024understanding}. In a few words, V-nets are neural networks designed to process arrays of vectors, thus treating multidimensional information as single entities. Hypercomplex-valued neural networks are particular V-nets enriched with geometric or algebraic properties. This paper presents theoretical results on the approximation capability of single hidden layer vector-valued MLP networks, supporting applications of V-nets, including hypercomplex-valued neural networks, for machine learning tasks.


{The paper is organized as follows: Section \ref{sec:background} briefly reviews concepts regarding algebras, including hypercomplex algebras. The main result of this work, namely, the universal approximation theorem for a broad class of vector-valued neural networks, is given in Section \ref{sec:UAT_HVNN}. Section \ref{sec:examples} presents numerical examples concerning the approximation capability of V-nets defined on two- and four-dimensional algebras. Some concluding remarks are given in Section \ref{sec:concluding}. The paper finishes with a brief review of the existing universal approximation for hypercomplex-valued neural networks in Appendix \ref{sec:UATS_diversos}.}

\section{Non-Degenerate and Hypercomplex Algebras} \label{sec:background}

Let us start by recalling the basic mathematical concepts necessary to introduce vector- and hypercomplex-valued neural networks \cite{valle2024understanding}.

Briefly, vector-valued neural networks (V-nets) are analogous to traditional neural networks, but the input, output, and some parameters are vectors instead of real scalars. The basic operations of addition and multiplication on vectors are defined in the scope of an algebra as follows.

\begin{definition}[Algebra \cite{Schafer1961AnAlgebras}]
\label{def:algebra}
An algebra $\mathbb{V}$ is a vector space over a field $\mathbb{F}$ with an additional bilinear operation called multiplication or product. The multiplication of $x \in \mathbb{V}$ and $y \in \mathbb{V}$ is denoted by the juxtaposition $x   y$.
\end{definition}

In this paper, we only consider finite-dimensional vector spaces, and $\mathcal{E}=\{\boldsymbol{e}_0,\ldots,\boldsymbol{e}_{n-1}\}$ denotes an ordered basis for $\mathbb{V}$. Furthermore, we only consider algebras over the field of real numbers, that is, $\mathbb{F}=\mathbb{R}$.

The basis $\mathcal{E}=\{\boldsymbol{e}_0,\ldots,\boldsymbol{e}_{n-1}\}$ allow us associating $x \in \mathbb{V}$ with an unique $n$-tuple $(\xi_0,\ldots,\xi_{n-1}) \in \mathbb{R}^n$ such that 
\begin{equation}
 x = \sum_{i=0}^{n-1} \xi_i \boldsymbol{e}_i.   
\end{equation}
Despite being equivalent, we use the isomorphism $\varphi:\mathbb{V} \to \mathbb{R}^n$ defined by
\begin{equation}
\label{eq:isomorphism}
  {\varphi(x) = (\xi_0, \ldots, \xi_{n-1})},
\end{equation}
to further distinguish {$x = \sum_{i=1}^{n-1} \xi_i \boldsymbol{e}_i \in \mathbb{V}$} from the $n$-tuple {$\varphi(x)=(\xi_0,\ldots,\xi_{n-1}) \in \mathbb{R}^n$} of the components of $x$ with respect to the basis $\mathcal{E}$. Using the isomorphism $\varphi$, $\mathbb{V}$ inherits the metric and topology from $\mathbb{R}^n$. In particular, the absolute value of $x \in \mathbb{V}$, denoted by $|x|$, is defined as follows using the Euclidean norm of  $\varphi(x)$, that is,
\begin{equation}
    \label{eq:absolute_value}
    |x| = \|\varphi(x)\|_2 = \sqrt{\sum_{i=0}^{n-1} {\xi_i}^2}.
\end{equation}
Furthermore, the component projection $\pi_i:\mathbb{V} \to \mathbb{R}$ given by 
\begin{equation}
\label{eq:projection}
    \pi_i(x) = {\xi_i}, \quad \forall i=0,\ldots,n-1,
\end{equation}
which maps $x = \sum_{i=0}^{n-1} {\xi_{i}} \boldsymbol{e}_i$ to the $i$th component ${\xi_i}$ of $\varphi(x)$, will be used to study the approximation capability of V-nets. Like the isomorphism $\varphi$, the component projection $\pi_i$ is linear, that is, $\pi_i(\alpha x+ y)=\alpha \pi_i(x)+\pi_i(y)$, for all $\alpha \in \mathbb{R}$ and $x,y \in \mathbb{V}$. 

{
Given an ordered basis $\basis$ for a vector space $\mathbb{V}$ over $\mathbb{R}$, a vector-valued function $f_{\mathbb{V}}:\mathbb{V}^N \to \mathbb{V}$, mapping the Cartesian product $\mathbb{V}^N = \mathbb{V}\times \mathbb{V} \times \ldots \times \mathbb{V}$ to $\mathbb{V}$, is given by
\begin{equation}
  \label{eq:pi-f}
  f_{\mathbb{V}}(\boldsymbol{x}) = \sum_{i=0}^{n-1} f_i(\boldsymbol{x}) \boldsymbol{e}_i, \quad \forall \boldsymbol{x}=(x_1,\ldots,x_N) \in \mathbb{V}^N,
\end{equation}
where $f_0,\ldots,f_n:\mathbb{V}^N \to \mathbb{R}$ are the (real-valued) components of $f_{\mathbb{V}}$. Note that the identity $f_i = \pi_i \circ f_{\mathbb{V}}$ holds for all $i=0,\ldots,n$.
}

Because $\mathbb{V}$ is a vector space, addition and multiplication by scalar are defined as follows for all 
$x={\sum_{i=0}^{n-1} \xi_i} \boldsymbol{e}_i \in \mathbb{V}$, $y = {\sum_{i=0}^{n-1}} \eta_i \boldsymbol{e}_i \in \mathbb{V}$, and $\alpha \in \mathbb{R}$:
\begin{equation}
    \label{eq:addition}
    x+y = \sum_{i=0}^{n-1} (\xi_i+\eta_i)\boldsymbol{e}_i \qeq \alpha x = \sum_{i=0}^{n-1} (\alpha \xi_i)\boldsymbol{e}_i.
\end{equation}
Being a bilinear operation, the multiplication of $x={\sum_{i=0}^{n-1}} \xi_i \boldsymbol{e}_i \in \mathbb{V}$ and $y = {\sum_{j=0}^{n-1}} \eta_j \boldsymbol{e}_j \in \mathbb{V}$ satisfies
\begin{align}\label{eq:multiplication}
x   y= \Bigg(\sum_{i=0}^{n-1}\xi_i\boldsymbol{e}_{i} \Bigg)\Bigg(\sum_{j=0}^{n-1}\eta_{j}\boldsymbol{e}_{j} \Bigg)= \sum_{i=0}^{n-1}\sum_{j=0}^{n-1} \xi_i\eta_{j}(\boldsymbol{e}_{i}   \boldsymbol{e}_{j}).
\end{align}
However, the multiplication of the basis elements $\boldsymbol{e}_{i}$ and $\boldsymbol{e}_{j}$ belongs $\mathbb{V}$. Thus, there exists scalar $p_{ij,k}$, for $k=0,\ldots,n-1$, such that $\boldsymbol{e}_{i}   \boldsymbol{e}_{j} =\sum_{k=0}^{n-1}p_{ij,k}\boldsymbol{e}_{k} $ and, from \eqref{eq:multiplication}, we obtain 
\begin{align}\label{mult.equation}
xy=  \sum_{i=0}^{n-1}\sum_{j=1}^{n-1}x_{j}\eta_{j} \Bigg(\sum_{k=0}^{n-1}p_{ij,k}\boldsymbol{e}_{k} \Bigg) = \sum_{k=0}^{n-1} \Bigg( \sum_{i=0}^{n-1}\sum_{j=0}^{n-1}\xi_i\eta_{j}p_{ij,k}\Bigg)\boldsymbol{e}_{k}. 
\end{align}
Defining the bilinear form $\mathcal{B}_{k}: \mathbb{V} \times \mathbb{V} \to \mathbb{R}$ given by
\begin{align}\label{eq:bilinear_form}
    \mathcal{B}_{k}(x,y)=\sum_{i=0}^{n-1}\sum_{j=0}^{n-1}\xi_i\eta_{j}p_{ij,k} \,\,\,\ \forall \, x,y \in \mathbb{V},
\end{align}
the multiplication of $x$ and $y$ given by \eqref{mult.equation} becomes
\begin{align}\label{mult.equation2}
x  y = \sum_{k=0}^{n-1}\mathcal{B}_{k}(x,y)\boldsymbol{e}_{k}.   
\end{align}
We would like to point out that \eqref{mult.equation2} is particularly useful for studying the approximation capability of V-nets. However, alternative expressions exist for computing the multiplication $x  y$ of elements $x,y \in \mathbb{V}$. From a computational point of view, in particular, the multiplication can be efficiently computed using Kronecker product \cite{Grassucci2022PHNNs:Convolutions,valle2024understanding}.

It is worth noting that the bilinear forms $\mathcal{B}_k$ defined by \eqref{eq:bilinear_form}, and thus the multiplication, are entirely characterized by the scalars $p_{ij,k}$ that appear in the product of the basis elements. Moreover, the matrix representation of the bilinear forms $\mathcal{B}_k$ with respect to the basis $\mathcal{E}=\{\boldsymbol{e}_0,\ldots,\boldsymbol{e}_{n-1}\}$ is 
\begin{equation}
{
    \label{eq:B_k}
    \boldsymbol{B}_k = \begin{bmatrix}
        p_{00,k} & p_{01,k} & \ldots & p_{0(n-1),k} \\
        p_{10,k} & p_{11,k} & \ldots & p_{1(n-1),k} \\
        \vdots & \vdots & \ddots & \vdots \\
        p_{(n-1)0,k} & p_{(n-1)1,k} & \ldots & p_{(n-1)(n-1),k}
    \end{bmatrix} \in \mathbb{R}^{n \times n}, \quad \forall k=0,\ldots,n-1.
    }
\end{equation}
Using the isomorphism $\varphi$ given by \eqref{eq:isomorphism}, we have 
\begin{eqnarray}
    \mathcal{B}_k(x,y) = \varphi(x)^T \boldsymbol{B}_k \varphi(y), \quad \forall x,y\in \mathbb{V}.
\end{eqnarray}

Next, we define the non-degeneracy of an algebra. From linear algebra, we know that a bilinear form $\mathcal{B}:\mathbb{V}\times \mathbb{V}\to \mathbb{R}$ is non-degenerate if and only if
\begin{equation}
 \mathcal{B}(x,y)=0, \; \forall x \in \mathbb{V} \iff y=0_{\mathbb{V}} \qeq
 \mathcal{B}(x,y)=0, \; \forall y \in \mathbb{V} \iff x=0_{\mathbb{V}},   
\end{equation}
where $0_{\mathbb{V}}$ denotes the zero of $\mathbb{V}$. A bilinear form that fails these conditions is called degenerate. Equivalently, a bilinear form $\mathcal{B}$ is non-degenerate if and only if its matrix representation $\boldsymbol{B} \in \mathbb{R}^{n \times n}$ is non-singular. Borrowing the terminology from linear algebra, we introduce the following definition:
\begin{definition}[Non-degenerate Algebra] 
\label{def:non-degenerate-algebra}
{An algebra $\mathbb{V}$ is non-degenerate with respect to the basis $\mathcal{E}=\{\boldsymbol{e}_0,\ldots,\boldsymbol{e}_{n-1}\}$} if all the bilinear forms in \eqref{eq:multiplication} are non-degenerate or, equivalently, the matrix $\boldsymbol{B}_k$ given by \eqref{eq:B_k} is non-singular for all $k=1,\ldots,n$. Otherwise, the algebra $\mathbb{V}$ is said to be {degenerate with respect to the basis $\mathcal{E}$.}
\end{definition}

We would like to emphasize that, according to Definition \ref{def:non-degenerate-algebra}, the degeneracy of an algebra depends on the basis -- precisely, on the multiplication of the basis elements. As a consequence, an algebra can be {degenerate with respect to a basis $\mathcal{E}=\{\boldsymbol{e}_0,\ldots,\boldsymbol{e}_{n-1}\}$} but {non-degenerate with respect to another basis $\mathcal{E}'=\{\boldsymbol{e}_0',\ldots,\boldsymbol{e}_{n-1}'\}$}. We illustrate this remark in Section \ref{sec:examples}.

\subsection{Hypercomplex Algebras}

Despite the general framework provided previously in this section, hypercomplex algebras play a crucial role partially due to their many successful applications, including their increasing interest in machine learning and deep learning \cite{Grassucci2022PHNNs:Convolutions,parcollet2020survey,Vieira2022AcuteNetworks}. In general terms, a hypercomplex algebra is an algebra whose multiplication has additional algebraic or geometrical property \cite{valle2024understanding}. Precisely, let us consider the following definition, which encompasses complex numbers, quaternions, Clifford algebras \cite{vaz16}, Caylay-Dickson algebras \cite{Schafer}, and the general framework provided by Kantor and Solodovnik \cite{kantor1989hypercomplex}.

\begin{definition}[Hypercomplex Algebra] \label{def:hypercomplex-algebra}
A hypercomplex-algebra, denoted by $\mathbb{H}$, is a finite-dimensional algebra endowed with a two-sided identity.    
\end{definition}

Accordingly, an element $\boldsymbol{1}_{\mathbb{H}} \in \mathbb{H}$ is a two-sided identity if $x = \boldsymbol{1}_{\mathbb{H}}   x = x    \boldsymbol{1}_{\mathbb{H}}$ for all $x \in \mathbb{H}$. Moreover, if a two-sided identity exists, it is unique. The identity $\boldsymbol{1}_{\mathbb{H}}$ is usually taken as the first element of an ordered basis. The canonical basis of a hypercomplex algebra $\mathbb{H}$ of dimension $\text{dim}(\mathbb{H}) = n$ is denoted by $\tau = \{\boldsymbol{1}_{\mathbb{H}},\ii_1,\ldots,\ii_{n-1}\}$ in this paper. 
Using the canonical basis, a hypercomplex number $x \in \mathbb{H}$ is written as
\begin{eqnarray}\label{eq:hyper_number}
x=\xi_{0}+\xi_{1} \boldsymbol{i}_{1}+\ldots+\xi_{n-1} \boldsymbol{i}_{n-1},
\end{eqnarray}
where $\xi_{0}, \xi_{1}, \ldots, \xi_{n-1} \in \mathbb{R}$ and, as usual in hypercomplex algebras, the identity $\boldsymbol{1}_{\mathbb{H}}$ has been omitted. The elements $\boldsymbol{i}_{1}, \boldsymbol{i}_{2}, \ldots, \boldsymbol{i}_{n-1}$ are called hyperimaginary units. 
From \eqref{eq:bilinear_form}, the product of two hypercomplex numbers $x \in \mathbb{H}$ and $y \in \mathbb{H}$ satisfies the identity
\begin{eqnarray}\label{eq5}
x  y= \mathcal{B}_{0}(x,y) + \sum_{k=1}^{n-1} \mathcal{B}_{k}(x,y) {\ii_k},
\end{eqnarray}
where $\mathcal{B}_{0}, \mathcal{B}_{1},\ldots, \mathcal{B}_{n-1}: \mathbb{H} \times \mathbb{H}\to \mathbb{R}$ are bilinear forms whose matrix representations in the canonical basis $\tau=\{\boldsymbol{1}_{\mathbb{H}}, \boldsymbol{i}_{1}, \cdots, \boldsymbol{i}_{n-1}\}$ are
\begin{equation}
\label{eq:hyper_B0}
{\boldsymbol{B}_0 = \begin{bmatrix}
1 & 0 & \cdots & 0\\
0 & p_{11,0} & \cdots & p_{1(n-1),0}\\
\vdots & \vdots & \ddots & \vdots\\
0 & p_{(n-1)1,0} & \cdots & p_{(n-1)(n-1),0}
\end{bmatrix} \in \mathbb{R}^{n\times n},}
\end{equation}
and, for $k=1, \dots,n-1$,
\begin{equation}
\label{eq:hyper_Bk}
{
\boldsymbol{B}_k = \begin{bmatrix}
0 & 0 & 0 & \cdots & 1 & \cdots & 0 &\\
0 & p_{11,k} & p_{12,k} &\cdots & p_{1j,k} & \cdots & p_{1(n-1),k} & \\
\vdots & \vdots & \vdots & \vdots & \vdots & & \vdots &\\
1 & p_{j1,k} &p_{j2,k} & \cdots & p_{jj,k} & \cdots & p_{j(n-1),k} & \\
\vdots & \vdots & \vdots & \vdots & \vdots & & \vdots & \\
0 & p_{(n-1)1,k} & p_{(n-1)2,k}& \cdots & p_{(n-1)j,k} & \cdots & p_{(n-1)(n-1),k}
\end{bmatrix}  \in \mathbb{R}^{n \times n}.
}
\end{equation}
Note that the first row and the first column of the matrix $\boldsymbol{B}_k$ correspond to the $(k+1)$th row and column of the $n\times n$ identity matrix, respectively, for any $k=0,1,\ldots,n-1$.
According to Definition \ref{def:non-degenerate-algebra}, {a hypercomplex algebra is non-degenerate with respect to the canonical basis $\tau=\{\boldsymbol{1}_{\mathbb{H}}, \boldsymbol{i}_{1}, \cdots, \boldsymbol{i}_{n-1}\}$} if the matrices $\boldsymbol{B}_0,\ldots,\boldsymbol{B}_{n-1}$ given by \eqref{eq:hyper_B0} and \eqref{eq:hyper_Bk} are all non-singular.

Complex numbers, quaternions, and octonions are examples of hypercomplex algebras. Hyperbolic numbers, dual numbers, and tessarines are also hypercomplex algebras. Examples of algebras, with a focus on hypercomplex algebras, are given in Section \ref{sec:examples}.

\subsection{Representation of Linear Functionals on Non-degenerate Algebras}

Our contribution to the approximation capability of vector- and hypercomplex-valued neural networks is based on the research conducted by Arena et al. \cite{ARENA1997335}. Accordingly, in the following, we expand Lemma 3.1 of \cite{ARENA1997335} from quaternions to an arbitrary non-degenerate algebra $\mathbb{V}$.

\begin{lemma}[Representation of Linear Functionals on Non-degenerate Algebras]\label{lem:functional}
Let $\mathbb{V}$ be a {non-degenerate algebra with respect to a basis $\basis$} and $\pi_i:\mathbb{V}\to \mathbb{R}$ denote a component projection given by \eqref{eq:projection} for $i \in \{0,1,\ldots,n-1\}$. Given a linear functional $\mathcal{L}:\mathbb{R}^{nN} \to \mathbb{R}$, there exist $y_1,\ldots,y_N \in \mathbb{V}$, {which depend of the component projection $\pi_i$,} such that 
\begin{equation}
\label{eq:linear_functional}
(\mathcal{L} \circ \varphi)(\boldsymbol{x}) = \pi_i \left( {\sum_{j=1}^{N}} y_j   x_{j} \right), \quad \forall \boldsymbol{x}=(x_1,\ldots,x_N) \in \mathbb{V}^N,    
\end{equation}
where $\varphi:\mathbb{V}^N \to \mathbb{R}^{nN}$ is obtained by applying the isomorphism $\varphi:\mathbb{V} \to \mathbb{R}^n$ given by \eqref{eq:isomorphism} component-wise. 
\end{lemma}

\begin{remark}
    Because $\mathcal{L}:\mathbb{R}^{nN}\to \mathbb{R}$ and $\varphi:\mathbb{V}^N \to \mathbb{R}^{nN}$ are both linear, their composition $\mathcal{L}\circ \varphi:\mathbb{V}^N \to \mathbb{R}$ is also a linear functional.
\end{remark}

\begin{proof}
Let $L(\mathbb{R}^{nN},\mathbb{R})$ denote the set of all linear functions from $\mathbb{R}^{nN}$ to $\mathbb{R}$, that is, 
\begin{equation}
 L(\mathbb{R}^{nN},\mathbb{R}) = \{\mathcal{L}:\mathbb{R}^{nN} \to \mathbb{R}:\mathcal{L} \; \text{is a linear functional}\}.   
\end{equation}
Note that
\begin{eqnarray}
\dim\big(L(\mathbb{R}^{nN},\mathbb{R})\big)=\dim(\mathbb{R}^{nN})\dim(\mathbb{R})=nN.    
\end{eqnarray}
Thus, using the isomorphism $\varphi:\mathbb{V}^N\to \mathbb{R}^{nN}$, we conclude that $\mathbb{V}^N$ and $L(\mathbb{R}^{nN},\mathbb{R})$ have the same dimension.

Let $T:\mathbb{V}^N \to L(\mathbb{R}^{nN},\mathbb{R})$ be the mapping defined by the following equation for all $\boldsymbol{y}=(y_1,\ldots,y_N) \in \mathbb{V}^N$:
\begin{equation}
    T(\boldsymbol{y})\big(\varphi(\boldsymbol{x})\big) = \pi_i\left(\sum_{j=1}^N y_j   x_j\right), \quad \forall \boldsymbol{x}=(x_1,\ldots,x_N) \in \mathbb{V}^N.
\end{equation}
In other words, $T(\boldsymbol{y})=\mathcal{L}$, where $\mathcal{L}:\mathbb{R}^{nN} \to \mathbb{R}$ is the linear functional given by \eqref{eq:linear_functional}. 
We will show that $T$ is a linear one-to-one mapping.

The linearity of $T$ is a straightforward consequence of the linearity of the component projection $\pi_i$. Let us now show that $T$ is an injective mapping or, equivalently, that $\ker(T)=\{\boldsymbol{0}\} \subset \mathbb{V}^N$.

From \eqref{eq:multiplication}, the following identities hold true:
\begin{equation}
    T(\boldsymbol{y})\big(\varphi(\boldsymbol{x})\big) = \pi_i\left(\sum_{j=1}^N y_j   x_j\right) = \sum_{j=1}^N \pi_i(y_j   x_j) = \sum_{j=1}^N \mathcal{B}_i(y_j,x_j).
\end{equation}
Because the algebra is {non-degenerate with respect to the basis $\mathcal{E}=\{\boldsymbol{e}_0,\ldots,\boldsymbol{e}_{n-1}\}$}, by hypothesis, the bilinear form $\mathcal{B}_i:\mathbb{V} \times \mathbb{V} \to \mathbb{R}$ is non-degenerate. Thus, we have
\begin{equation}
    T(\boldsymbol{y})\big(\varphi(\boldsymbol{x})\big) = \sum_{j=1}^N \mathcal{B}_i(y_j,x_j)=0_{\mathbb{R}},\quad \forall \boldsymbol{x}=(x_1,\ldots,x_N)\in \mathbb{V}^N,
\end{equation}
if and only if $y_j=0_{\mathbb{V}}$ for all $j=1,\ldots,N$. Hence, we have $\ker(T)=\{\boldsymbol{0}\}$.

Finally, since $T:\mathbb{V}^N \to L(\mathbb{R}^{nN},\mathbb{R})$ is a linear injective mapping between spaces with the same dimension, $T$ is a one-to-one linear mapping.

Concluding, given a linear functional $\mathcal{L}:\mathbb{V}^N \to \mathbb{R}$, there exists $\boldsymbol{y}=(y_1,\ldots,y_N)\in \mathbb{V}^N$ such that $T(\boldsymbol{y})=\mathcal{L}$; namely $\boldsymbol{y} = T^{-1}(\mathcal{L})$. Equivalently, there exist $y_1,\ldots,y_N \in \mathbb{V}^{N}$ such that \eqref{eq:linear_functional} holds true.
\end{proof}

Before proceeding to this paper's main result, we would like to emphasize that Lemma \ref{lem:functional} holds for non-degenerate algebras {with respect to a given basis $\basis$}. The following example illustrates that we may fail to represent a linear functional on a degenerate algebra. Precisely, the following example shows that we cannot represent a given linear function as the component projection of a weighted sum in the hypercomplex algebra of dual numbers {with the canonical basis}.

\begin{example}
Let $\mathbb{D}$ denote the hypercomplex algebra of dual numbers, which is an {algebra degenerate with respect to the canonical basis $\tau = \{1,\ii\}$}, where $\ii^2=0$. Accordingly, the multiplication of two dual numbers $x=\xi_0 + \xi_1 \ii$ and $y = \eta_0 + \eta_1 \ii$ yields
\begin{eqnarray}
 x  y = \xi_0\eta_0 + (\xi_0 \eta_1 + \xi_1 \eta_0)\ii.   
\end{eqnarray}
A linear functional $\mathcal{L}:\mathbb{R}^4 \to \mathbb{R}$ is given by $\mathcal{L}(t_1,t_2,t_3,t_4)=at_1+bt_2+ct_3+dt_4$, for some $a,b,c,d \in \mathbb{R}$. Using the isomorphism $\varphi:\mathbb{D}^2 \to \mathbb{R}^4$, we obtain
    \begin{equation} \label{eq:example_LinearDual}
        (\mathcal{L}\circ \varphi)(\vetx)= a\xi_{1,0}+b\xi_{1,1}+c\xi_{2,0}+d\xi_{2,1}, \quad \forall \vetx=(x_1,x_2) \in \mathbb{D}^2,
    \end{equation}
    where $x_1 = \xi_{1,0}+\xi_{1,1}\ii \in \mathbb{D}$ and $x_2 = \xi_{2,0}+\xi_{2,1}\ii \in \mathbb{D}$. Moreover, we have
    \begin{eqnarray}
        \pi_0(y_1   x_1+y_2   x_2)=\eta_{1,0}\xi_{1,0}+\eta_{2,0}\xi_{2,0},
    \end{eqnarray}
    where $\pi_0:\mathbb{D} \to \mathbb{R}$ denotes the projection into the real part of a dual number, that is, $\pi_0(x)=\xi_0$ for all $x=\xi_0 + \xi_1\ii$.
    Hence, there exist $y_1 = \eta_{1,0}+\eta_{1,1}\ii \in \mathbb{D}$ and $y_2 = \eta_{2,0} + \eta_{2,1}\ii \in \mathbb{D}$ such that 
    \begin{equation}
        \label{eq:ex1_L}
        (\mathcal{L}\circ \varphi)(\vetx) = \pi_0(y_1   x_1 + y_2   x_2),
    \end{equation}
    if, and only if, $b=d=0$. Therefore, there do not exist $y_1,y_2 \in \mathbb{D}$ such that \eqref{eq:ex1_L} holds true for an arbitrary linear functional $\mathcal{L}:\mathbb{R}^4 \to \mathbb{R}$.
\end{example}

\section{Universal Approximation Theorem for V-Nets} 

\label{sec:UAT_HVNN}


This section extends the universal approximation theorem to a broad class of vector-valued multilayer perceptron (V-MLP) networks with split activation functions. {However, before presenting the result, let us first review the class of Tauber-Wiener functions \cite{chen1995universal}. Additionally, we will formalize the concept of split activation functions in an algebra $\mathbb{V}$.}

{
The Tauberian theorems, which are named after Alfred Tauber, are essential in various areas of mathematics, especially mathematical analysis \cite{Korevaar2004TauberianTheory}. Many prominent researchers have contributed to the development of Tauberian theorems, including Norbert Wiener with his particularly comprehensive work  \cite{Wiener1932TauberianTheorems}. Among other contributions, Wiener's work provides interesting results on the approximation capability of linear combinations of translations of a given function. Motivated by Tauber and Wiener's works, Chen et al. introduced a class of functions that have been used to prove the approximation capability of real-valued MLP networks \cite{chen1995universal,chen1990constructive}.
}

\begin{definition}[Tauber-Wiener Functions]
{
A real-valued function $\psi_{\mathbb{R}}:\mathbb{R}\to \mathbb{R}$ is called a Tauber-Wiener function if, given a continuous function $f_{\mathbb{R}}:\mathbb{R}\to \mathbb{R}$, a closed interval $[a,b]$, and $\varepsilon>0$, there exists an integer $M>0$ and parameters $\alpha_i \in \mathbb{R}$, $w_i \in \mathbb{R}$, and $b_i \in \mathbb{R}$, for  $i=1,\ldots,M$, such that 
\begin{equation}
    \left|f_{\mathbb{R}}(x)-\sum_{i=1}^M \alpha_i \psi_{\mathbb{R}}(w_ix+b_i)\right| \leq \varepsilon, \quad
    \forall x \in [a,b].
\end{equation}
}
\end{definition}

In other words, any continuous function can be approximated by a linear combination of translations, compressions, or stretching of a Tauber-Wiener function. Note, however, that such a linear combination corresponds to a single-hidden layer MLP network with a Tauber-Wiener activation function $\psi_{\mathbb{R}}$ that maps $\mathbb{R}$ to $\mathbb{R}$. Moreover, Chen et al. as well as Leshno et al. proved independently that the MLP network's approximation capability can be extended from $\mathbb{R} \to \mathbb{R}$ to $\mathbb{R}^N \to \mathbb{R}$ \cite{chen1995universal,chen1990constructive,leshno1993multilayer}. Hence, Tauber-Wiener functions are necessary and sufficient for the MLP network's approximation capability of continuous functions. Examples of Tauber-Wiener functions include discriminatory functions \cite{Cybenko1989ApproximationFunction}, extended sigmoid functions \cite{chen1990constructive,chen1995universal}, and, more generally, non-polynomial continuous functions like the ReLU activation function \cite{leshno1993multilayer,Pinkus1999ApproximationNetworks}. In the following, we formulate the universal approximation theorem for V-MLP networks using split Tauber-Wiener activation functions.

\begin{definition}[Split Activation Function]
Let $\mathcal{E} = \{\boldsymbol{e}_{0},\ldots,\boldsymbol{e}_{n-1} \}$ be an ordered basis for a finite-dimension vector space $\mathbb{V}$ over $\mathbb{R}$. A split activation function $\psi_{\mathbb{V}}: \mathbb{V} \to \mathbb{V}$ derived from a real-valued function $\psi_{\mathbb{R}}: \mathbb{R} \to \mathbb{R}$ is defined by
\begin{align}\label{splitfunction}
\psi_{\mathbb{V}}(x) = \sum_{i=0}^{n-1} \psi_{\mathbb{R}}(\xi_i)\boldsymbol{e}_{i}, \quad  \forall x = \sum_{i=0}^{n-1} x_{i}\boldsymbol{e}_{i} \in \mathbb{V}.
\end{align}
\end{definition}

\begin{remark}
    The split activation function $\psi_{\mathbb{V}}$ and its generator $\psi_\mathbb{R}$ are related by the isomorphism $\varphi:\mathbb{V} \to \mathbb{R}^n$ defined by \eqref{eq:isomorphism} through the identity
    \begin{equation}
        (\varphi \circ \psi_{\mathbb{V}})(x) = (\psi_{\mathbb{R}} \circ \varphi)(x), \quad \forall x \in \mathbb{V},
    \end{equation}
    where $\psi_{\mathbb{R}}$ is extended component-wise to $\mathbb{R}^n$. In other words, the column vector representation of a split activation function's outcome equals the real-valued function evaluated on the argument's column vector representation. In particular, we have 
    \begin{eqnarray}
        (\pi_i \circ \psi_{\mathbb{V}})(x) = (\psi_{\mathbb{R}}\circ \pi_i)(x), \quad \forall i=0,,\ldots,n-1,
    \end{eqnarray}
    where $\pi_i:\mathbb{V}\to \mathbb{R}$ denotes the component projection given by \eqref{eq:projection}.
\end{remark}

Examples of split activation functions include the split $\mathtt{ReLU}$ and the split sigmoid functions, which are usually recommended in practical applications and yield approximation capabilities.

Broadly speaking, a vector-valued neural network (V-net) is a neural network whose inputs, outputs, and trainable parameters are elements of a finite-dimensional algebra. Hypercomplex-valued neural networks are V-nets defined on an algebra with multiplication identity. 
By making such a general definition, we encompass previously known models such as complex, quaternion, hyperbolic, tessarine, and Clifford-valued networks as particular instances, thus resulting in a broader family of neural network models. {Furthermore, because the product in an algebra $\mathbb{V}$ is characterized by bilinear forms, V-nets are equivalent to the arbitrary bilinear product neural networks (ABIPNN) introduced recently by Fan et al. \cite{Fan2020BackpropagationProducts}.}

In the following, we formalize the concept of single hidden layer V-MLP networks. We also define the particular class of single hidden layer V-MLP networks with scalar output weights. 

\begin{definition}[Single Hidden Layer V-MLP network]\label{def:HMLP}
Let $\mathbb{V}$ be a finite-dimensional algebra over $\mathbb{R}$. A single hidden layer vector-valued multilayer perceptron (V-MLP)  network -- with vector-valued output weights -- defines a mapping $\mathcal{N}_{\mathbb{V}}:\mathbb{V}^N \to \mathbb{V}$ as follows 
\begin{equation}\label{eq:V-MLP}
\mathcal{N}_\mathbb{V}(\boldsymbol{x}) = \sum_{i=1}^{M} \alpha_{i} \psi_{\mathbb{V}}\left(\sum_{j=1}^N w_{ij}x_j + b_{i}\right), \quad \forall \boldsymbol{x} = (x_1,\ldots,x_N) \in \mathbb{V}^N,
\end{equation}
where $w_{ij} \in \mathbb{V}$ and $\alpha_{i} \in \mathbb{V}$ are the weights between input and hidden layers, and hidden and output layers, respectively, for all $i=1,\ldots,M$ and $j=1,\ldots,N$. Furthermore, the parameters $b_1,\ldots,b_M \in \mathbb{V}$ represent the biases for the neurons in the hidden layer and $\psi_{\mathbb{V}}: \mathbb{V} \to \mathbb{V}$ denotes the activation function. 

A single hidden layer V-MLP network with scalar output weights is also given by \eqref{eq:V-MLP} with $w_{ij} \in \mathbb{V}$ and $b_i \in \mathbb{V}$, but the weights of the output layer are such that $\alpha_i \in \mathbb{R}$ for all $i=1,\ldots,M$. 
\end{definition}

Note that the V-MLP network given by \eqref{eq:V-MLP} has $M$ hidden neurons. Moreover, the output $y \in \mathcal{N}_{\mathbb{V}}(\boldsymbol{x}) \in \mathbb{V}$ even when $\alpha_i \in \mathbb{R}$ for all $i=1,\ldots,M$. Hence, the V-MLP with scalar output weights also yields an element of $\mathbb{V}$.

Let us now present the universal approximation theorem for V-MLPs with scalar output weights. The following theorem also shows that single hidden layer V-MLP with vector-valued output weights can approximate continuous function if the algebra $\mathbb{V}$ has an identity. In other words, hypercomplex-valued MLP networks with hypercomplex-valued output weights are also universal approximators.


\begin{theorem}\label{thm:tau_geral}
Let $\psi_{\mathbb{R}}: \mathbb{R} \to \mathbb{R}$ be a Tauber-Wiener function such that {$\lim_{t \to -\infty} \psi_{\mathbb{R}}(t) = 0$ or $\lim_{t \to +\infty} \psi_{\mathbb{R}}(t) = 0$ hold true.} Also, let $\mathbb{V}$ be a {non-degenerate algebra with respect to $\basis$}, $\psi_{\mathbb{V}}: \mathbb{V} \to \mathbb{V}$ be the split activation function derived from $\psi_{\mathbb{R}}$, and $K \subset \mathbb{V}^N$ be a compact set. The class of all single hidden layer V-MLP networks with scalar output weights defined by
\begin{align}\label{eq.33}
\mathcal{H}_{\mathbb{V}} = \bigg\{ \mathcal{N}_{\mathbb{V}}(\boldsymbol{x}) =\sum_{i=1}^{M} \alpha_{i} \psi_{\mathbb{V}}\bigg( \sum_{j=1}^{N} w_{ij}x_{j} + b_{i} \bigg): M \in \mathbb{N}, \alpha_{i} \in \mathbb{R}, w_{ij}, b_{i} \in \mathbb{V} \bigg\},
\end{align}
is dense in the  set $\mathcal{C}(K)$ of all continuous functions from $K$ to $\mathbb{V}$. Furthermore, if $\mathbb{V} \equiv \mathbb{H}$ is a hypercomplex algebra, then the class of single hidden layer hypercomplex-valued MLP networks defined by
\begin{align}
\mathcal{H}_{\mathbb{H}} = \bigg\{ \mathcal{N}_{\mathbb{H}}(\boldsymbol{x}) =\sum_{i=1}^{M} \alpha_{i} \psi_{\mathbb{V}}\bigg( \sum_{j=1}^{N} w_{ij}x_{j} + b_{i} \bigg): M \in \mathbb{N}, \alpha_{i}, w_{ij}, b_{i} \in \mathbb{H} \bigg\},
\end{align}
is dense in $\mathcal{C}(K)$.
\end{theorem}

\begin{proof}[Proof of Theorem \ref{thm:tau_geral}]
{
First of all, recall that $\mathcal{H}_{\mathbb{V}}$ is dense in $C(K)$ if, and only if, given a continuous vector-valued function $f_{\mathbb{V}}:K\subset \mathbb{V}^N \to \mathbb{V}$ and $\varepsilon>0$, there exists a V-MLP network $\mathcal{N}_{\mathbb{V}}:\mathbb{V}^N \to \mathbb{V}$ given by \eqref{eq:V-MLP} with $w_{ij} \in \mathbb{V}$, $b_i \in \mathbb{V}$, and $\alpha_i \in \mathbb{R}$, for all $i=1,\ldots,M$ and $j=1,\ldots,N$ such that
\begin{equation}
|\mathcal{N}_{\mathbb{V}}(\vetx)-f_{\mathbb{V}}(\vetx)|\leq \varepsilon,  \quad \forall \vetx \in K.  
\end{equation}
In the following, we use the approximation property of real-valued MLP networks to construct such a V-MLP network.
}

{
Let $f_{\mathbb{V}}:\mathbb{V}^N \to \mathbb{V}$ be a given continuous vector-valued function and consider $\varepsilon>0$. The vector-valued function $f_{\mathbb{V}}$ can be written as follows
\begin{eqnarray}
    f_{\mathbb{V}}(\vetx) = \sum_{k=0}^{n-1} f_k(\vetx)\boldsymbol{e}_k, \quad \forall \vetx=(x_1,\ldots,x_N) \in K \subset \mathbb{V}^N,
\end{eqnarray}
where $f_k:K \to \mathbb{R}$ denotes the $k$th componet of $f_{\mathbb{V}}$, for all $k=0,\ldots,n-1$. Because $\psi_{\mathbb{R}}:\mathbb{R} \to \mathbb{R}$ is a Tauber-Wiener function, it yields approximation capability to real-valued MLP networks (see Theorem 3 in \cite{chen1995universal} or Step 2 in the proof of Theorem 1 in \cite{leshno1993multilayer}). Thus, {using the isomorphism $\varphi:\mathbb{V}^N \to \mathbb{R}^{nN}$}, there exist real-valued MLP networks $\mathcal{N}_k:\mathbb{R}^{nN}\to \mathbb{R}$ such that
\begin{equation}
    \label{eq:N_k2f_k}
    |\mathcal{N}_k\big(\varphi(\vetx)\big)-f_k(\vetx)| \leq \frac{\varepsilon}{2 \sqrt{n}}, 
\end{equation}
for all $\vetx \in K$ and $k=0,\ldots,n-1$.
Equivalently\footnote{{See Appendix \ref{ssec:approximation-traditional-networks} for further details on the approximation capability of traditional neural networks.}}, there exist $M_k \in \mathbb{N}$, $\alpha_{ik}\in \mathbb{R}$, $b_{ik}^{(\mathbb{R})} \in \mathbb{R}$, and linear functionals $\mathcal{L}_{ik}:\mathbb{R}^{nN}\to \mathbb{R}$ such that
\begin{eqnarray}
    \left|\sum_{i=1}^{M_k}{\alpha_{ik}}
    \psi_{\mathbb{R}}\left(\mathcal{L}_{ik}\big( \varphi(\vetx)\big)+b_{ik}^{(\mathbb{R})}\right) -f_k(\vetx)\right| \leq \frac{\varepsilon}{2 \sqrt{n}},
\end{eqnarray}
for all $\vetx \in K$ and $k=0,\ldots,n-1$. By hypothesis, the algebra $\mathbb{V}$ is {non-degenerate with respect to $\basis$}. From Lemma \ref{lem:functional}, there exist parameters $w_{ijk} \in \mathbb{V}$ such that
\begin{eqnarray}
    \left|\sum_{i=1}^{M_k}{\alpha_{ik}} \psi_{\mathbb{R}}\left(\pi_k\left(\sum_{j=1}^N w_{ijk} x_j\right)+b_{ik}^{(\mathbb{R})}\right) -f_k(\vetx)\right| \leq \frac{\varepsilon}{2 \sqrt{n}},
\end{eqnarray}
We define the V-MLP network by the following equation for all $\vetx=(x_1,\ldots,x_N)\in \mathbb{V}^N$,
\begin{eqnarray} \label{eq:V-MLP_proof}
    \mathcal{N}_{\mathbb{V}}(\vetx) = \sum_{k=0}^{n-1} \sum_{i=1}^{M_k} \alpha_{ik} \psi_{\mathbb{V}}\left( \sum_{j=1}^N w_{ijk}x_j + b_{ik}\right), 
\end{eqnarray}
where the bias terms $b_{ik} \equiv b_{ik}(\lambda) \in \mathbb{V}$ are defined as follows using a parameter $\lambda \in \mathbb{R}$:
\begin{equation}
    b_{ik} = b_{ik}^{(\mathbb{R})}\boldsymbol{e}_k + \sum_{\ell \neq k}^{n-1} \lambda \boldsymbol{e}_\ell.
\end{equation}
Note that \eqref{eq:V-MLP_proof} is equivalent to \eqref{eq:V-MLP} if we combine the two sums ``$\sum_{k=0}^{n-1}$'' and ``$\sum_{i=1}^{M_k}$'' into a single sum ``$\sum_{i=1}^M$'', with $M=M_0+\ldots+M_{n-1}$, and redefine the indexes accordingly. Also, note that the following identity holds for all $\ell=0,\ldots,n-1$:
\begin{equation} \label{eq:theta_ik}
    \pi_\ell(b_{ik}) =\begin{cases}
        b_{ik}^{(\mathbb{R})}, & \ell = k,\\
        \lambda, & \text{otherwise}.
    \end{cases}
\end{equation}
}

{We shall conclude the proof by showing that 
\begin{equation}
  \left|\mathcal{N}_{\mathbb{V}}(\vetx)-f_{\mathbb{V}}(\vetx)\right|\leq \varepsilon, \quad \forall \vetx \in K.  
\end{equation}
Before, however, let us rewrite $\mathcal{N}_{\mathbb{V}}$ using the real-valued MLP networks $\mathcal{N}_k:\mathbb{R}^{nN} \to \mathbb{R}$, for $k=0,\ldots,n-1$, which approximate the components of $f_{\mathbb{V}}$. Accordingly, from \eqref{eq:pi-f},  recalling that $\psi_{\mathbb{V}}:\mathbb{V} \to \mathbb{V}$ is a split activation function derived from $\psi_{\mathbb{R}}:\mathbb{R} \to \mathbb{R}$, and using \eqref{eq:theta_ik}, we have
\begin{align}
    \mathcal{N}_{\mathbb{V}}(\vetx) 
    &= \sum_{k=0}^{n-1} \sum_{i=1}^{M_k} \alpha_{ik} \left[\sum_{\ell=0}^{n-1} (\pi_\ell \circ \psi_{\mathbb{V}})\left(\sum_{j=1}^N w_{ijk}x_j + b_{ik}\right) \boldsymbol{e}_\ell \right] \\
    &= \sum_{k=0}^{n-1} \sum_{i=1}^{M_k} \alpha_{ik} \left[\sum_{\ell=0}^{n-1} ( \psi_{\mathbb{R}} \circ \pi_\ell)\left(\sum_{j=1}^N w_{ijk}x_j + b_{ik}\right) \boldsymbol{e}_\ell \right] \\
    &= \sum_{k=0}^{n-1} \sum_{i=1}^{M_k} \alpha_{ik} \left[\psi_{\mathbb{R}} \left(\pi_k\left(\sum_{j=1}^N w_{ijk}x_j\right)+b_{ik}^{(\mathbb{R})}\right) \boldsymbol{e}_k \right.  \label{eq:xx} \\
    & \nonumber \qquad + \left.\sum_{\ell\neq k}^{n-1} \psi_{\mathbb{R}}\left( \pi_\ell \left(\sum_{j=1}^N w_{ijk}x_j\right) + \lambda\right) \boldsymbol{e}_\ell \right] \\
    &= \sum_{k=0}^{n-1} (\mathcal{N}_k\circ\varphi)(\vetx) \boldsymbol{e}_k + \mathcal{R}(\vetx,\lambda), \label{eq:yy}
\end{align}
where 
\begin{eqnarray}
    (\mathcal{N}_k \circ \varphi)(\vetx) = \sum_{i=1}^{M_k} \alpha_{ik} \psi_{\mathbb{R}}\left( \mathcal{L}_{ik}\big(\varphi(\vetx)\big)+b_{ik}^{(\mathbb{R})}\right),
\end{eqnarray}
corresponds to the real-valued MLP network that approximates the component $f_k$ of $f_{\mathbb{V}}$ and 
\begin{eqnarray}
    \mathcal{R}(\vetx,\lambda) = \sum_{k=0}^{n-1} \sum_{i=1}^{M_k} \alpha_{ik} \sum_{\ell \neq k}^{n-1} \psi_{\mathbb{R}}\left( \pi_\ell\left( \sum_{j=1}^N w_{ijk} x_j \right)+\lambda\right) \boldsymbol{e}_\ell \in \mathbb{V},
\end{eqnarray}
is a remainder term that depends on $\vetx \in \mathbb{V}^N$ and $\lambda \in \mathbb{R}$. Being $K \subset \mathbb{V}$ a compact and $\pi_\ell:\mathbb{V} \to \mathbb{R}$ continuous, the set
\begin{equation}
    B = \left\{ \pi_\ell\left(\sum_{j=1}^N w_{ijk}x_j \right): \vetx \in K\right\} \subset \mathbb{R},
\end{equation}
is bounded by the Weierstrass extreme value theorem. By hypothesis, the activation function $\psi_{\mathbb{R}}$ satisfies {$\lim_{t \to -\infty}\psi_{\mathbb{R}}(t)=0$ or $\lim_{t \to +\infty}\psi_{\mathbb{R}}(t)=0$. Thus, the remainder $\mathcal{R}(\vetx,\lambda)$ approaches zero as $\lambda$ increases or decreases, depending on the asymptotic behavior of $\psi_{\mathbb{R}}$.} Formally, there exists $\Lambda > 0$ such that
\begin{eqnarray}
    \label{eq:R_0}
    |\mathcal{R}(\vetx,\lambda)|\leq \frac{\varepsilon}{2}, \quad \forall \vetx \in K \; {\text{and either}\; \lambda \leq -\Lambda \text{ or } \lambda \geq \Lambda}.
\end{eqnarray}
Concluding, using the triangle inequality, the identity \eqref{eq:absolute_value}, and the inequalities \eqref{eq:R_0} and \eqref{eq:N_k2f_k}, respectively, we obtain 
\begin{align}
    \left| \mathcal{N}_{\mathbb{V}}(\vetx)-f_{\mathbb{V}}(\vetx)\right| & = 
    \left| \sum_{k=0}^{n-1} (\mathcal{N}_k\circ\varphi)(\vetx) \boldsymbol{e}_k + \mathcal{R}(\vetx,\lambda) - \sum_{k=0}^{n-1}f_k(\vetx)\boldsymbol{e}_k \right| \\
    &\leq \left| \sum_{k=0}^{n-1} \left( (\mathcal{N}_k\circ\varphi)(\vetx) - f_k(\vetx)\right)\boldsymbol{e}_k\right| + \left|\mathcal{R}(\vetx,\lambda)\right| \\
    & \leq \sqrt{{n}\frac{\varepsilon^2}{{4n}}}+ \frac{\varepsilon}{2} = \varepsilon,
\end{align}
for all $\vetx \in K$ and {either $\lambda \leq -\Lambda$ or $\lambda \geq \Lambda$}.
}

{
Finally, let us prove the second part of the theorem. Because $\mathcal{H}_{\mathbb{V}}$ is dense in the set $\mathcal{C}(K)$ of all continuous functions from $K \subseteq \mathbb{V}^N$ to $\mathbb{R}$, given $f_{\mathbb{V}}:\mathbb{V}^N \to \mathbb{V}$ and $\varepsilon>0$, there exists a V-MLP $\mathcal{N}_{\mathbb{V}}:\mathbb{V}^N \to \mathbb{V}$ given by \eqref{eq:V-MLP} with $M \in \mathbb{N}$, $\alpha_i^{(\mathbb{R})} \in \mathbb{R}$, $w_{ij} \in \mathbb{V}$ and $b_i \in \mathbb{V}$, for all $i=1,\ldots,M$ and $j=1,\ldots,N$, such that $|\mathcal{N}_{\mathbb{V}}(\vetx)-f_{\mathbb{V}}(\vetx)| \leq \varepsilon$ for all $\vetx \in K$. However, if $\mathbb{V} \equiv \mathbb{H}$ is a hypercomplex algebra, then there exists $\boldsymbol{1}_{\mathbb{H}}$ such that $\boldsymbol{1}_{\mathbb{H}}x = x$ for all $x \in \mathbb{H}$. Thus, defining $\alpha_i = \alpha_i^{(\mathbb{R})} \boldsymbol{1}_{\mathbb{H}} \in \mathbb{H}$, we conclude that the hypercomplex-valued neural network $\mathcal{N}_{\mathbb{H}}:\mathbb{H}^N \to \mathbb{H}$ given by \eqref{eq:V-MLP} satisfies the following identities:
\begin{equation}
    \mathcal{N}_{\mathbb{H}}(\vetx) = \sum_{i=1}^M \alpha_i \psi\left(\sum_{j=1}^N w_{ij}x_j +b_j \right) = \sum_{i=1}^M \alpha_i^{(\mathbb{R})} \psi\left(\sum_{j=1}^N w_{ij}x_j +b_j \right).
\end{equation}
Therefore, it is equivalent to the V-MLP with real-valued output weights that approximate the hypercomplex-valued continuous function $f_{\mathbb{H}}$.
}
\end{proof}

We would like to point out that our proof is based on the universal approximation theorem for quaternion-valued MLP networks with scalar output weights by Arena et al. \cite{ARENA1997335}. However, we generalize Arena and collaborators' results in two folds: First, we generalize the approximation theorem for general V-MLP networks with scalar output weights. Second, we show that the approximation capability holds for hypercomplex-valued MLP networks with hypercomplex output weights. 

\begin{remark}
{Theorem \ref{thm:tau_geral} may be expanded to include a wider range of V-MLP networks with vector-valued output weights, provided that additional assumptions about the algebra are imposed. These assumptions should justify, in particular, the transition from \eqref{eq:xx} to \eqref{eq:yy} by considering the multiplication of the basis elements.}
\end{remark}






\section{Numerical Examples}
\label{sec:examples}

Let us now illustrate the approximation capability of MLP networks using vector- and hypercomplex-valued algebras in two and four dimensions.

\subsection{Numerical Example with Two-Dimensional Algebras} \label{sec:2D_Examples}

Complex, hyperbolic, and dual numbers are hypercomplex algebras of dimension 2, i.e., the elements of these algebras are of the form {$x = \xi_{0} + \xi_{1}\ii$}. They differ in the value of $\ii^2$, the square of the hyperimaginary unit. The most well-known of these two-dimensional (2D) hypercomplex algebras is the complex numbers where $\ii^2 = -1$. Complex numbers are effectively used in signal processing, physics, electromagnetism, and electrical and electronic circuits. In contrast, the hyperbolic unit satisfies $\ii^2 = 1$. Hyperbolic numbers have important connections with abstract algebra, ring theory, and special relativity \cite{Catoni2008TheSpace-Time}. Lastly, dual numbers are hypercomplex algebra in which the imaginary unit satisfies $\ii^2 = 0$.

{
In the most general case, the elements of a two-dimensional algebra can be written as {$x= \xi_0 \boldsymbol{e}_0 + \xi_1 \boldsymbol{e_1}$}, where $\xi_0,\xi_1 \in \mathbb{R}$ are scalars and $\boldsymbol{e}_0,\boldsymbol{e}_1$ are the basis elements. The product of two basis elements satisfy
$\boldsymbol{e}_i \boldsymbol{e}_j = p_{ij,0} \boldsymbol{e}_0 + p_{ij,1} \boldsymbol{e}_1$, for all $i,j \in \{0,1\}$. From \eqref{eq:bilinear_form}, the product of $x$ and $y$ satisfies $xy = \mathcal{B}_0(x,y)\boldsymbol{e}_0 + \mathcal{B}_1(x,y)\boldsymbol{e}_1$, where the matrices associated with the bilinear forms $\mathcal{B}_0$ and $\mathcal{B}_1$ are given by
\begin{equation}
    \label{eq:matrices-2D}
    \boldsymbol{B}_0 = \begin{bmatrix}
        p_{00,0} & p_{01,0} \\ p_{10,0} & p_{11,0}
    \end{bmatrix} \qeq
    \boldsymbol{B}_1 = \begin{bmatrix}
        p_{00,1} & p_{01,1} \\ p_{10,1} & p_{11,1}
    \end{bmatrix},
\end{equation}
respectively. Table \ref{tab:2D} contains examples of the matrices $\boldsymbol{B}_0$ and $\boldsymbol{B}_1$ of some two-dimensional algebras. 
}

{
We would like to point out that Table \ref{tab:2D} contains known algebras, such as the complex numbers ($\mathbb{C}$) and dual numbers ($\mathbb{D}$), as well as new algebras introduced for illustrative purposes. Namely, the algebras denoted by $\mathbb{A}$ and $\mathbb{B}$ have no identity; thus, they are examples of algebras that are not hypercomplex. Furthermore, the algebra entitled ``equivalent to dual numbers'', denoted by $\mathbb{E}$, is obtained from the dual numbers by defining $\boldsymbol{e}_0 = 1+\ii$ and $\boldsymbol{e}_1 = 1-\ii$, with $\ii^2=0$. 
}

{
Note from Table \ref{tab:2D} that the algebra denoted by $\mathbb{A}$, the complex numbers ($\mathbb{C}$), and the algebra equivalent to the dual numbers are non-degenerate. In contrast, the dual numbers ($\mathbb{D}$) and the algebra denoted by $\mathbb{B}$ are degenerate because their matrices $\boldsymbol{B}_0$ and $\boldsymbol{B}_1$, respectively, are singular. It is worth noting that a degenerate algebra (e.g., dual numbers, $\mathbb{D}$) can be transformed into a non-degenerate algebra (the algebra $\mathbb{E}$ equivalent to the dual numbers) by a change of basis. Thus, the degeneracy of an algebra is not an invariant concept; it can be avoided by choosing appropriate basis elements. Also, note that the matrices $\boldsymbol{B}_0$ and $\boldsymbol{B}_1$ of the non-degenerate algebra $\mathbb{A}$ in Table \eqref{tab:2D} coincide with the identity matrix. Hence, the multiplication of $x=\xi_0\boldsymbol{e}_0+\xi_1\boldsymbol{e}_1$ and $y=\eta_0\boldsymbol{e}_0+\eta_1\boldsymbol{e}_1$ yields $\mathcal{B}(x,y)(\boldsymbol{e}_0+\boldsymbol{e}_1)$, where $\mathcal{B}(x,y)=\xi_0\eta_0+\xi_1\eta_1$ agrees with the usual inner product. Despite the components being tied in the outcome of the multiplication, neural networks in this algebra exhibit approximation capability, as we will see below.
}

\begin{table}
    \label{tab:2D}
    \centering
    \begin{tabular}{c|cc}
       \textbf{Algebra}  & $\boldsymbol{B}_0$ & $\boldsymbol{B}_1$ \\ \toprule
    Non-degenerate Algebra ($\mathbb{A}$) & 
    $\begin{bmatrix} 1 & 0 \\ 0 & 1 \end{bmatrix}$ &
    $\begin{bmatrix} 1 & 0 \\ 0 & 1 \end{bmatrix}$ \\ \midrule
    Degenerate Algebra ($\mathbb{B}$) & 
    $\begin{bmatrix} 1 & 0 \\ 1 & 1 \end{bmatrix}$ &
    $\begin{bmatrix} 1 & 1 \\ 1 & 1 \end{bmatrix}$ \\ \midrule   
    Complex Numbers ($\mathbb{C}$) & 
    $\begin{bmatrix} 1 & 0 \\ 0 & -1 \end{bmatrix}$ &
    $\begin{bmatrix} 0 & 1 \\ 1 & 0 \end{bmatrix}$ \\ \midrule
    Dual Numbers ($\mathbb{D}$) & 
    $\begin{bmatrix} 1 & 0 \\ 0 & 0 \end{bmatrix}$ &
    $\begin{bmatrix} 0 & 1 \\ 1 & 0 \end{bmatrix}$ \\ \midrule
    Equivalent to Dual Numbers ($\mathbb{E}$) & 
    $\begin{bmatrix} 3/2 & 1/2 \\ 1/2 & -1/2 \end{bmatrix}$ &
    $\begin{bmatrix} -1/2 & 1/2 \\ 1/2 & 3/2 \end{bmatrix}$ \\ \midrule
    \end{tabular} 
    \caption{Matrices associated with the bilinear form of some 2D algebras.}
    \label{tab:2D_algebras}
\end{table}

Let us illustrate Theorem \ref{thm:tau_geral} with a numerical example based on the two-dimensional algebras listed in Table \ref{tab:2D_algebras}. To this end, let $\mathcal{E}=\{\boldsymbol{e}_0,\boldsymbol{e}_1\}$ denote a basis for a two-dimensional algebra $\mathbb{V}$ and consider the non-linear continuous function $f_{\mathbb{V}}:\mathbb{V} \to \mathbb{V}$ defined by
\begin{equation}
    \label{eq:f-example-2D}
    f_{\mathbb{V}}(x) = (\xi_0^2 - \xi_1^2)\boldsymbol{e}_0 + (\xi_0^2+\xi_0\xi_1 + \xi_1^2)\boldsymbol{e}_1,\quad \forall x \in K,
\end{equation}
where 
\begin{equation}
    K = \{x=\xi_0\boldsymbol{e}_0+\xi_1\boldsymbol{e}_1:-1\leq \xi_0 \leq +1, -1\leq \xi_1 \leq +1\},
\end{equation}
is a compact subset of $\mathbb{V}$. Note that the components of $f_{\mathbb{V}}$ are quadratic forms, which are continuous but non-linear, on the components $x_0$ and $x_1$ of $x=\xi_0 \boldsymbol{e}_0+\xi_1\boldsymbol{e}_1$. The non-linear function $f_{\mathbb{V}}$ given by \eqref{eq:f-example-2D} serves as a simple example to illustrate the approximation capability of vector- and hypercomplex-valued neural networks because it fulfills the hypothesis of Theorem \ref{thm:tau_geral} but it poses mild difficulty to be approximated.

\begin{figure}
    \centering
    \includegraphics[width=\columnwidth]{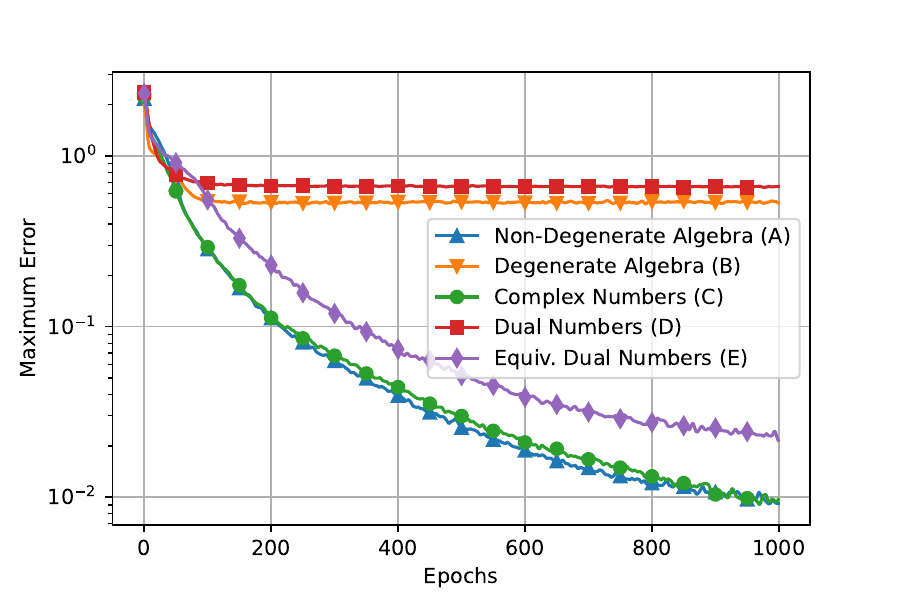}
    \caption{MSE by the number of epochs during the training phase of V-MLP networks with real-valued output weights defined on two-dimensional algebras.}
    \label{fig:Experiment-2D-R}
\end{figure}

In order to validate Theorem \ref{thm:tau_geral}, we implemented V-MLP networks using the five algebras shown in Table \ref{tab:2D_algebras}. The networks have real-valued output weights and 128 hidden neurons with the split \texttt{ReLU} activation function\footnote{Note that the rectified linear unit activation function satisfies $\lim_{t \to -\infty} \mathtt{ReLU}(t)=0$.}. They have been trained using a dataset $\mathcal{T}_{\mathbb{V}} = \{(x_i,f_{\mathbb{V}}(x_i)):i=1,\ldots,1024\} \subset \mathbb{V}\times \mathbb{V}$ consisting of 1024 samples taken from $K$ using a uniform distribution. Furthermore, the five V-MLP networks have been trained by minimizing the mean-squared error (MSE) using the Adam optimizer for 1000 epochs. {
For comparison purposes, we also implemented a traditional (real-valued) MLP network with 128 hidden neurons but trained on the dataset $\mathcal{T}_{\mathbb{R}}=\{\big(\varphi(x_i),\varphi(f_{\mathbb{V}}(x_i))\big):i=1,\ldots,1024\} \subset \mathbb{R}^2 \times \mathbb{R}^2$ derived from $\mathcal{T}_{\mathbb{V}}$ using the isomorphism $\varphi$. Figure \ref{fig:Experiment-2D-R} shows the MSE produced by the V-MLP networks and the traditional model by the number of epochs during the training phase.
}

Note from Figure \ref{fig:Experiment-2D-R} that the MSE decreased consistently for the networks that were based on non-degenerate algebras, namely the networks defined on the algebras $\mathbb{A}$, $\mathbb{C}$ (complex numbers), and $\mathbb{E}$ (the algebra equivalent to dual numbers). {In contrast, the MSE produced by the V-MLP networks based on dual-numbers $\mathbb{D}$ and the degenerate algebra $\mathbb{B}$ stagnated at the values $4.21\times 10^{-2}$ and $1.35\times 10^{-2}$, respectively.} This numerical example confirms the approximation capability of V-MLP networks defined on non-degenerate algebras. It also reveals that Theorem \ref{thm:tau_geral} may fail for a V-MLP network defined on a degenerate algebra.

\begin{figure}
    a) Surfaces of the components of the vector-valued function $f_{\mathbb{V}}$.
    \[\includegraphics[width=\columnwidth]{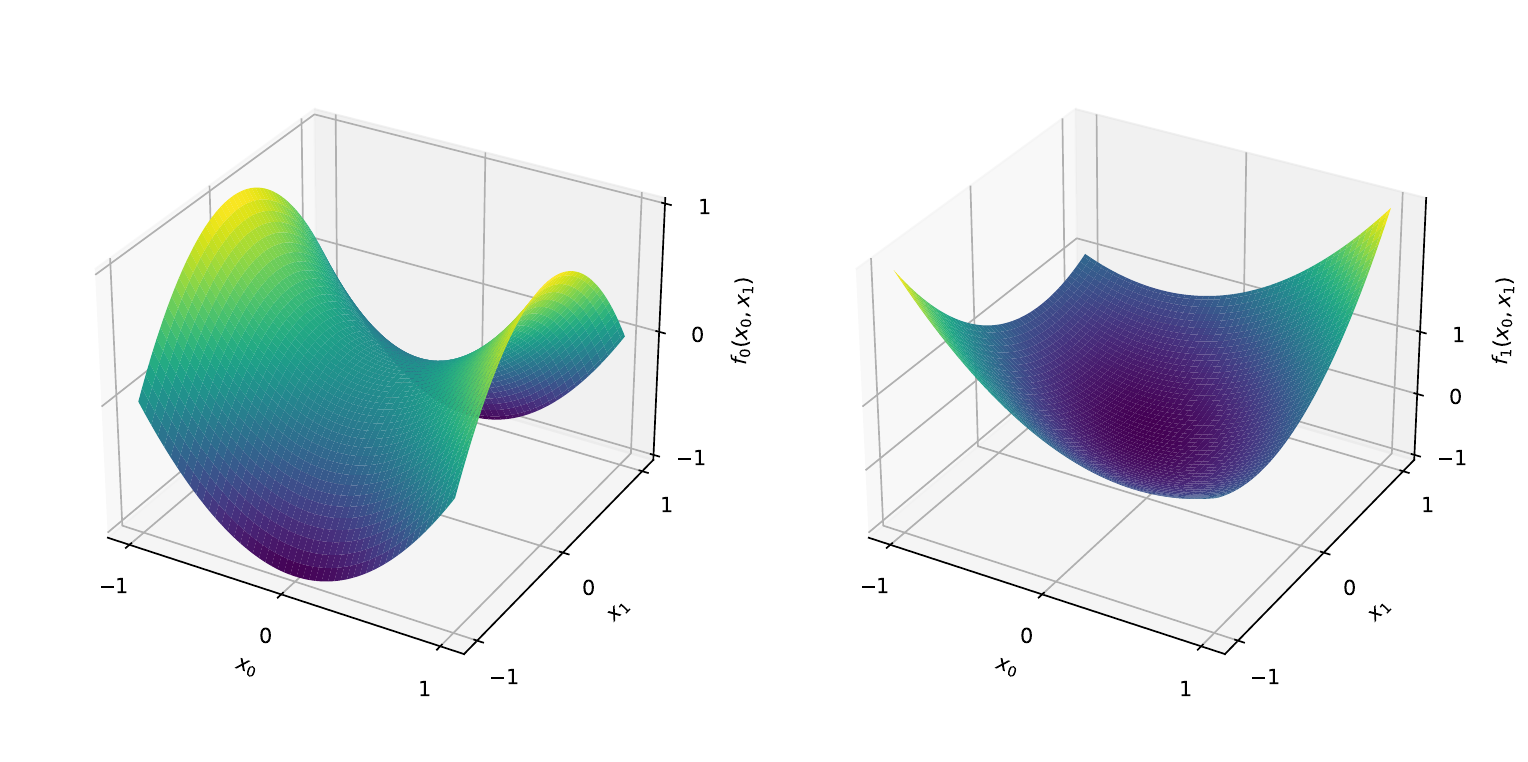}\]
    b) Surfaces of the components of the V-MLP network defined on the algebra $\mathbb{B}$. 
    \[\includegraphics[width=\columnwidth]{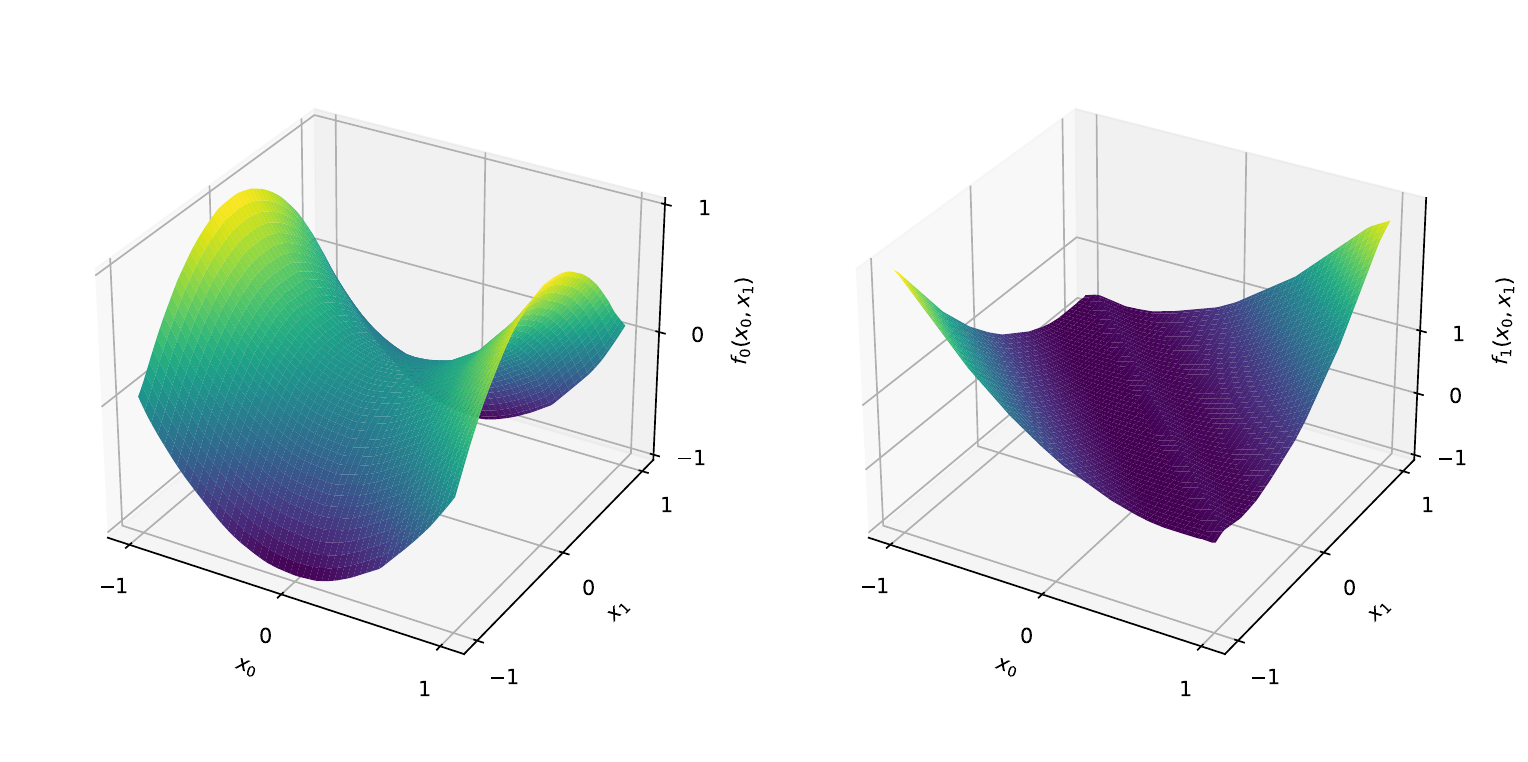}\]
    \caption{Surfaces of the components of the function $f_{\mathbb{V}}$ given by \eqref{eq:f-example-2D} and the V-MLP networks with real-valued output weights defined on a two-dimensional degenerate algebra.}
    \label{fig:Surfaces-2D-R1}
\end{figure}

\begin{figure}
    a) Surfaces of the components of the V-MLP network defined on dual numbers $\mathbb{D}$. 
    \[\includegraphics[width=\columnwidth]{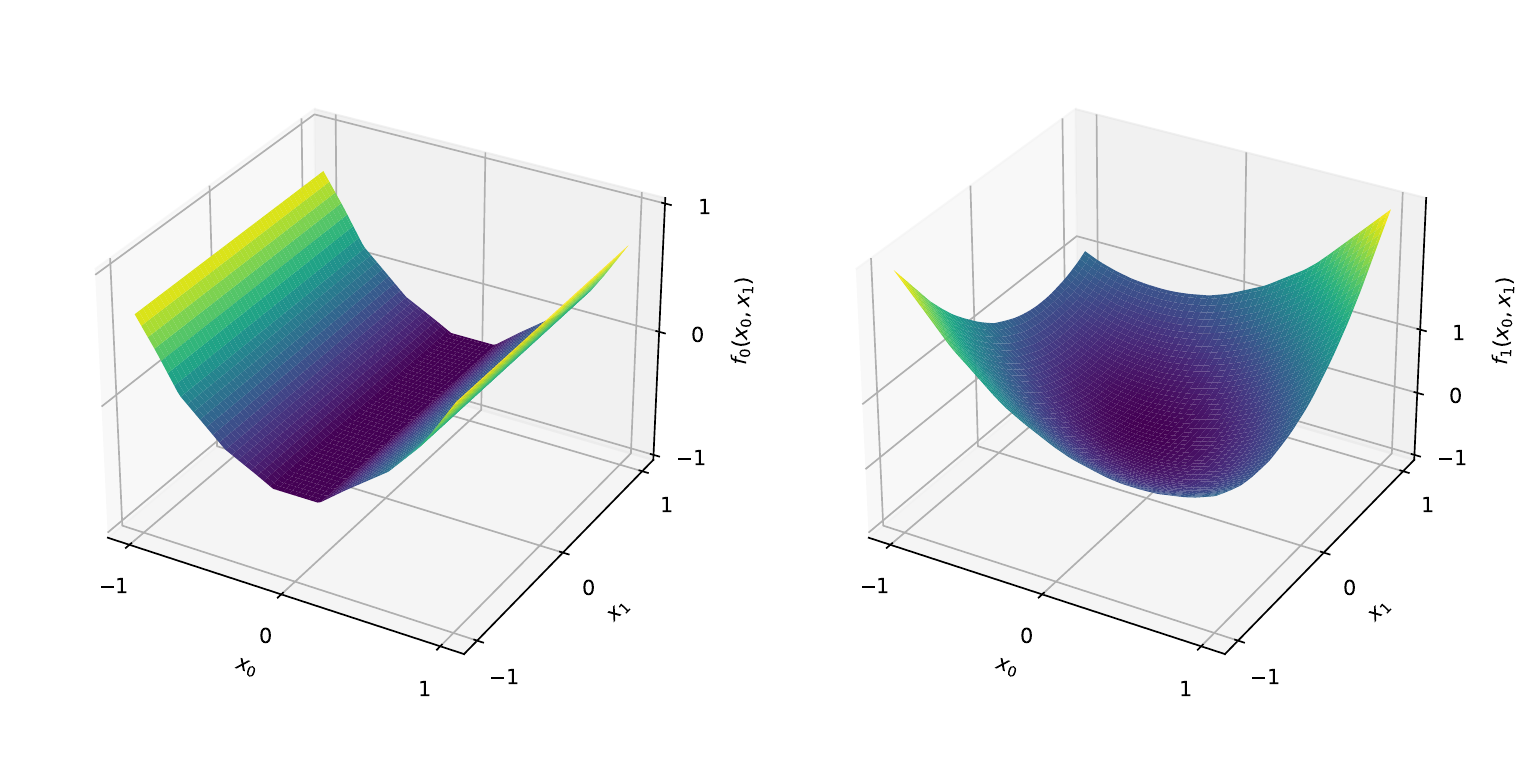}\]
    b) Surfaces of the components of the V-MLP network defined on the algebra $\mathbb{E}$. 
    \[\includegraphics[width=\columnwidth]{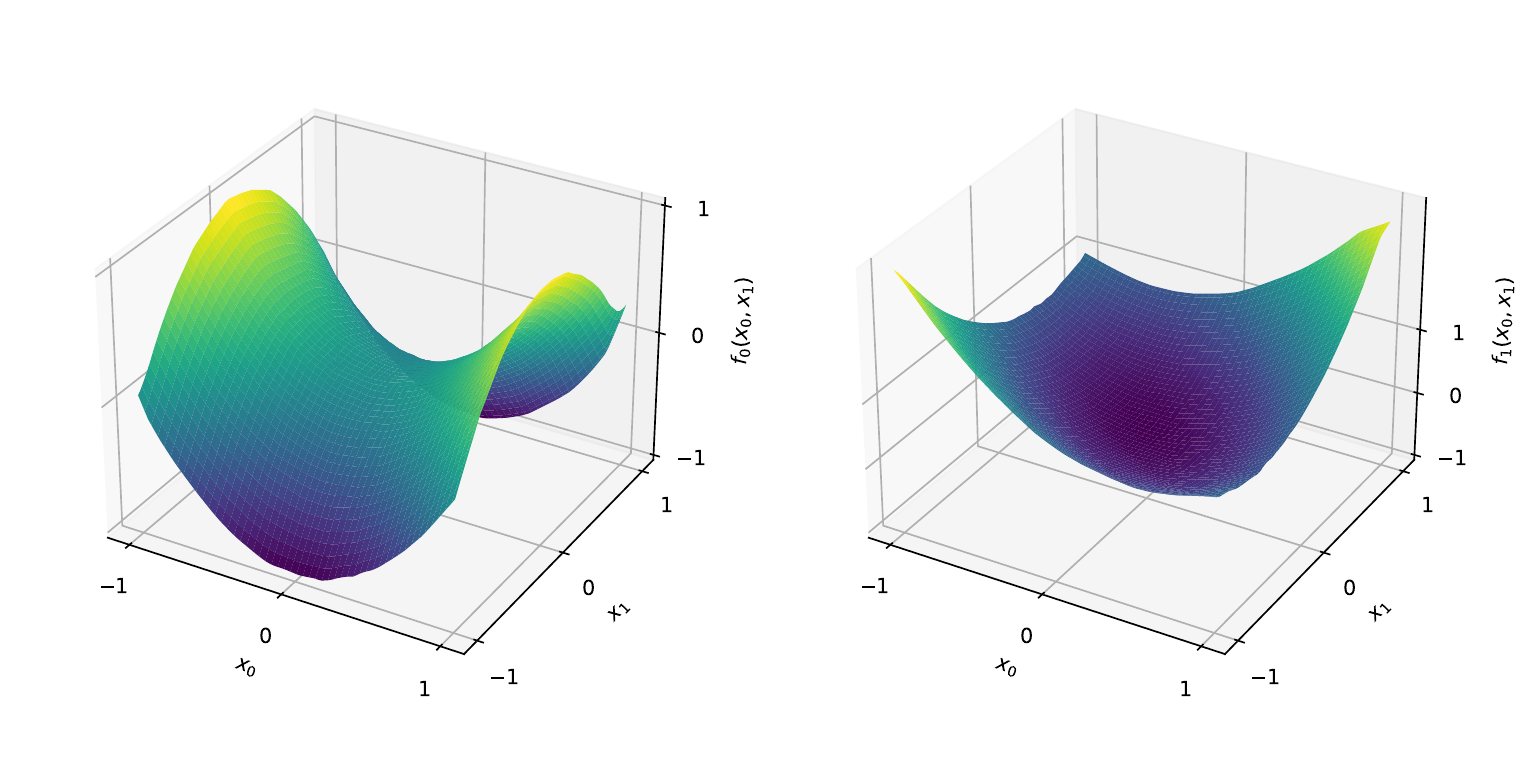}\]
    \caption{Surfaces of the V-MLP networks with real-valued output weights defined on the two-dimensional algebras $\mathbb{D}$ (dual numbers) and $\mathbb{E}$ (equivalent to dual numbers).}
    \label{fig:Surfaces-2D-R2}
\end{figure}

To better illustrate the approximation capability of the V-MLPs, Figures \ref{fig:Surfaces-2D-R1} and \ref{fig:Surfaces-2D-R2} display the surfaces associated with the components of $f_{\mathbb{V}}$ and the trained networks defined on the algebras $\mathbb{B}$, $\mathbb{D}$, and $\mathbb{E}$. Note that the V-MLP networks based on the algebra $\mathbb{E}$ visually coincide with the surfaces of $f_{\mathbb{V}}$, despite producing the largest MSE among the networks defined on non-degenerate algebras. However, the V-MLP networks based on degenerate algebras fail to approximate $f_{\mathbb{V}}$ precisely on the component whose bilinear form is degenerate. Accordingly, in the proof of Theorem \ref{thm:tau_geral}, we utilize the non-degeneracy of the bilinear form to determine the weights in the hidden layer of the V-MLP network. Conversely, the network may fail to approximate a component of the continuous function $f_{\mathbb{V}}$ if the corresponding bilinear form is degenerate.

\begin{figure}
    \centering
    \includegraphics[width=\columnwidth]{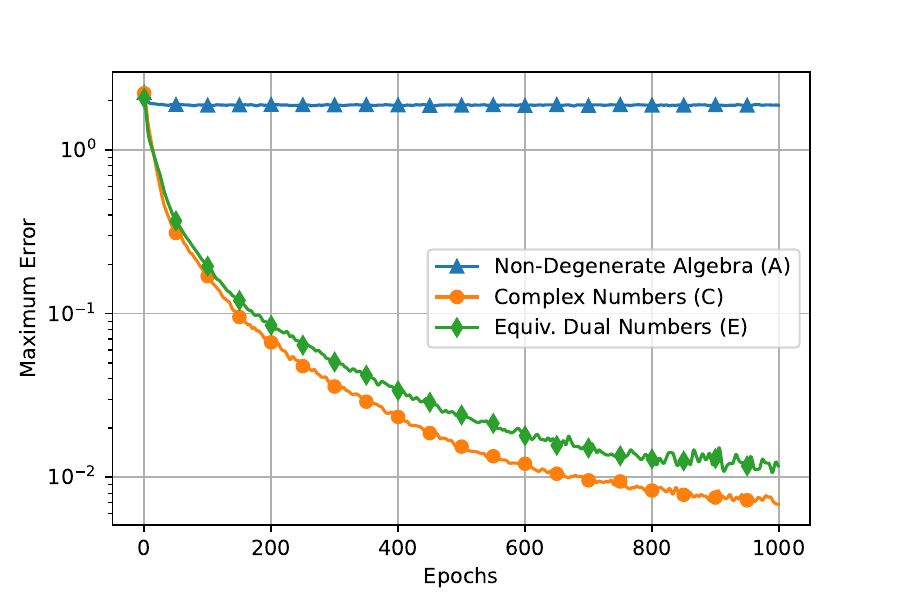}
    \caption{MSE by the number of epochs during the training phase of V-MLP networks with vector-valued output weights defined on non-degenerate two-dimensional algebras.}
    \label{fig:Experiment-2D-V}
\end{figure}

\begin{figure}
    \centering
    \includegraphics[width=\columnwidth]{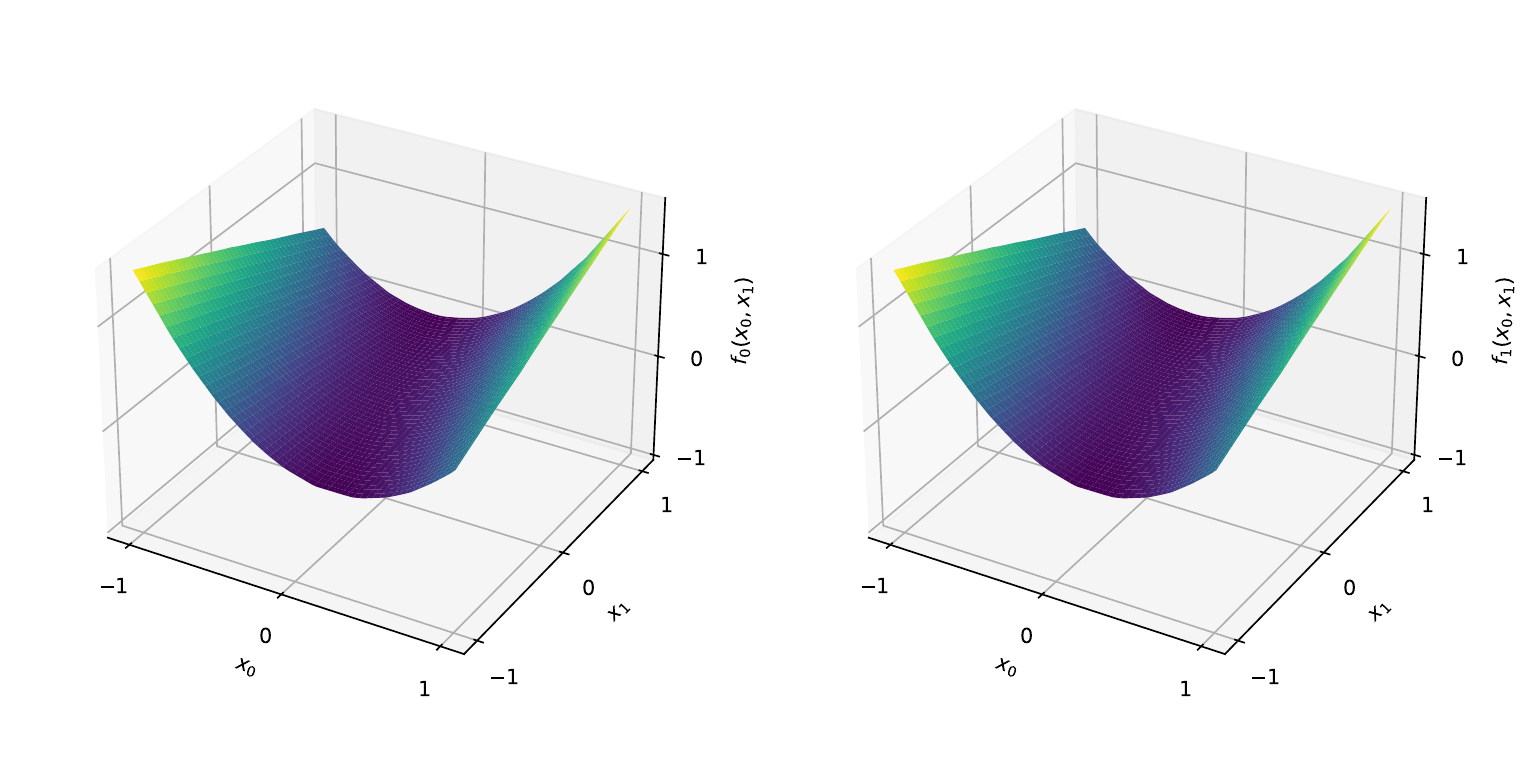}
    \caption{Surfaces of the components of the V-MLP network, with vector-valued output weights, defined on the two-dimensional non-degenerate algebra $\mathbb{A}$.}
    \label{fig:Surfaces-2D-V}
\end{figure}

Finally, to illustrate the second part of Theorem \ref{thm:tau_geral}, we approximated $f_{\mathbb{V}}$ using V-MLP networks with vector-valued output weights on non-degenerate algebras. We used the same network architecture (128 hidden neurons with split \texttt{ReLU} activation function) and training methodology (we minimized the MSE using the Adam method for 1000 epochs). Figure \ref{fig:Experiment-2D-V} shows the MSE by the number of epochs produced by the V-MLP network defined on the non-degenerate algebra $\mathbb{A}$, $\mathbb{C}$, and $\mathbb{E}$. 
{The MSE produced by the traditional MLP network is also depicted in Figure \ref{fig:Experiment-2D-V} as a reference model.}
Notice that the MSE produced by the MLP networks based on the non-degenerate hypercomplex algebras $\mathbb{C}$ and $\mathbb{E}$ decreased consistently. {However, the MSE produced by the network based on the non-degenerate algebra $\mathbb{A}$ (which is not a hypercomplex algebra) stagnated at $2.30\times 10^{-1}$.} Figure \ref{fig:Surfaces-2D-V} shows the surface of the components of the V-MLP defined on the algebra $\mathbb{A}$. As the bilinear forms associated with components of the product coincide in the algebra $\mathbb{A}$, the surfaces of the two components also coincide. 

In conclusion, this second numerical example confirms the approximation capability of MLP networks with hypercomplex-valued output weights defined on non-degenerate algebra. Furthermore, it provides an example of a V-MLP network with vector-valued output weights defined on an algebra without identity that is not able to approximate a given function. 

\subsection{Numerical Example with Four-Dimensional Algebras}
\label{sec:4D_Examples}


{
Let us now address four-dimensional (4D) algebras, which include quaternions \cite{parcollet2020survey}, tessarines \cite{Carniello2021UniversalNetworks,CERRONI2017232,Ortolani2017OnProcessing}, hyperbolic quaternions \cite{Takahashi2021ComparisonControl}, and Klein four-group \cite{huang2013klein,KOBAYASHI2020123}, and many other hypercomplex algebras as particular instances \cite{vieira2022general,Vieira2022AcuteNetworks}. Four-dimensional algebras are particularly interesting because they can be used to model control systems \cite{Takahashi2021ComparisonControl}, color and PolSAR images \cite{parcollet2020survey,shang14}, ambisonic signals \cite{Grassucci2023DualRepresentation}, and proved helpful in solving many image and signal processing tasks.
}

{
The elements in a 4D algebra are of the form $x=\xi_0 \boldsymbol{e}_0+\xi_1\boldsymbol{e}_1+\xi_2\boldsymbol{e}_2+\xi_3\boldsymbol{e}_3$. The multiplication of two 4D elements $x$ and $y$ is given by 
\begin{equation}\label{eq:prod4D}
    xy = \mathcal{B}_0(x,y)\boldsymbol{e}_0+\mathcal{B}_1(x,y)\boldsymbol{e}_1+
    \mathcal{B}_2(x,y)\boldsymbol{e}_2+\mathcal{B}_3(x,y)\boldsymbol{e}_3,
\end{equation}
where $\mathcal{B}_0,\ldots,\mathcal{B}_3$ are bilinear forms whose matrix representation with respect to the basis $\mathcal{E}=\{\boldsymbol{e}_0,\ldots,\boldsymbol{e}_3\}$ satisfies \eqref{eq:B_k}. Table \ref{tab:4D} provides examples of the matrices associated with the bilinear forms of some 4D algebras, including hypercomplex algebras such as quaternions ($\mathbb{Q}$), hyperbolic-quaternions ($\mathbb{HQ}$), and dual-complex numbers ($\mathbb{DC}$). More examples of 4D algebras can be found in \cite{Vieira2022AcuteNetworks,vieira2022general}. In particular, among other 4D hypercomplex algebras, the hyperbolic quaternions and dual-complex numbers have been used to design a servo-level robot manipulator controller by Takahashi \cite{Takahashi2021ComparisonControl}.  
}

{
The first algebra in Table \ref{tab:4D}, denoted by $\mathbb{F}$, is a non-degenerate 4D algebra in which all the bilinear forms coincide. In contrast, $\mathbb{G}$ is a degenerate algebra because the matrix $\boldsymbol{B}_3$ is singular. Neither $\mathbb{F}$ nor $\mathbb{G}$ have an identity element, making them not hypercomplex algebras. The quaternions, denoted by $\mathbb{Q}$, constitute one of the most well-known hypercomplex algebras. Due to their intrinsic relation between rotations in 3D space and the quaternion product, they have numerous applications, from physics to computer vision and control. Hyperbolic quaternions, introduced by MacFarlane, are an example of a non-associative hypercomplex algebra despite their algebraic similarity with quaternions (they differ only in the first bilinear form). Finally, dual-complex numbers are obtained by imposing that the real and the imaginary parts of a complex number are dual numbers. Dual-complex numbers constitute a degenerate algebra because the matrices $\boldsymbol{B}_0$ and $\boldsymbol{B}_1$ are singular. 
}

\begin{table}
    \label{tab:4D}
    \centering
    \begin{tabular}{c|cccc}
       \textbf{Algebra}  & $\boldsymbol{B}_0$ & $\boldsymbol{B}_1$ & $\boldsymbol{B}_2$ & $\boldsymbol{B}_3$ \\ \toprule
    \parbox[c]{7em}{\centering Non-Degenerate Algebra ($\mathbb{F}$)} & 
    {\small $\begin{bmatrix} 
    1 & 0 & 0 & 0 \\ 
    0 & 1 & 0 & 0 \\
    0 & 0 & 1 & 0 \\
    0 & 0 & 0 & 1 
    \end{bmatrix}$} &
    {\small $\begin{bmatrix} 
    1 & 0 & 0 & 0 \\ 
    0 & 1 & 0 & 0 \\
    0 & 0 & 1 & 0 \\
    0 & 0 & 0 & 1 
    \end{bmatrix}$} &
    {\small$\begin{bmatrix} 
    1 & 0 & 0 & 0 \\ 
    0 & 1 & 0 & 0 \\
    0 & 0 & 1 & 0 \\
    0 & 0 & 0 & 1 
    \end{bmatrix}$} & 
    {\small$\begin{bmatrix} 
    1 & 0 & 0 & 0 \\ 
    0 & 1 & 0 & 0 \\
    0 & 0 & 1 & 0 \\
    0 & 0 & 0 & 1 
    \end{bmatrix}$}
      \\ \midrule  
     \parbox[c]{7em}{\centering Degenerate Algebra ($\mathbb{G}$)} & 
     {\small$\begin{bmatrix} 
    1 & 0 & 0 & 0 \\ 
    0 & 1 & 0 & 0 \\
    0 & 0 & 1 & 0 \\
    0 & 0 & 0 & 1 
    \end{bmatrix}$} &
    {\small $\begin{bmatrix} 
    1 & 0 & 0 & 0 \\ 
    1 & 1 & 0 & 0 \\
    1 & 1 & 1 & 0 \\
    1 & 1 & 1 & 1 
    \end{bmatrix}$} &
    {\small $\begin{bmatrix} 
    1 & 1 & 1 & 1 \\ 
    0 & 1 & 1 & 1 \\
    0 & 0 & 1 & 1 \\
    0 & 0 & 0 & 1 
    \end{bmatrix}$} &
    {\small$\begin{bmatrix} 
    1 & 1 & 1 & 1 \\ 
    1 & 1 & 1 & 1 \\
    1 & 1 & 1 & 1 \\
    1 & 1 & 1 & 1 
    \end{bmatrix}$}
      \\ \midrule   
    \parbox[c]{7em}{\centering Quaternions ($\mathbb{Q}$)} & 
    {\small $\begin{bmatrix} 
    1 & 0 & 0 & 0 \\ 
    0 & -1 & 0 & 0 \\
    0 & 0 & -1 & 0 \\
    0 & 0 & 0 & -1 
    \end{bmatrix}$} &
    {\small $\begin{bmatrix} 
    0 & 1 & 0 & 0 \\ 
    1 & 0 & 0 & 0 \\
    0 & 0 & 0 & 1 \\
    0 & 0 & -1 & 0 
    \end{bmatrix}$} &
    {\small$\begin{bmatrix} 
    0 & 0 & 1 & 0 \\ 
    0 & 0 & 0 & -1 \\
    1 & 0 & 0 & 0 \\
    0 & 1 & 0 & 0 
    \end{bmatrix}$} & 
    {\small$\begin{bmatrix} 
    0 & 0 & 0 & 1 \\ 
    0 & 0 & 1 & 0 \\
    0 & -1 & 0 & 0 \\
    1 & 0 & 0 & 0 
    \end{bmatrix}$} \\ \midrule
    \parbox[c]{7em}{\centering Hyperbolic Quaternions ($\mathbb{HQ}$)} & 
    {\small $\begin{bmatrix} 
    1 & 0 & 0 & 0 \\ 
    0 & 1 & 0 & 0 \\
    0 & 0 & 1 & 0 \\
    0 & 0 & 0 & 1 
    \end{bmatrix}$} &
    {\small $\begin{bmatrix} 
    0 & 1 & 0 & 0 \\ 
    1 & 0 & 0 & 0 \\
    0 & 0 & 0 & 1 \\
    0 & 0 & -1 & 0 
    \end{bmatrix}$} &
    {\small$\begin{bmatrix} 
    0 & 0 & 1 & 0 \\ 
    0 & 0 & 0 & -1 \\
    1 & 0 & 0 & 0 \\
    0 & 1 & 0 & 0 
    \end{bmatrix}$} & 
    {\small$\begin{bmatrix} 
    0 & 0 & 0 & 1 \\ 
    0 & 0 & 1 & 0 \\
    0 & -1 & 0 & 0 \\
    1 & 0 & 0 & 0 
    \end{bmatrix}$}
      \\ \midrule
    \parbox[c]{7em}{\centering Dual-Complex Numbers ($\mathbb{DC}$)} & 
    {\small $\begin{bmatrix} 
    1 & 0 & 0 & 0 \\ 
    0 & -1 & 0 & 0 \\
    0 & 0 & 0 & 0 \\
    0 & 0 & 0 & 0 
    \end{bmatrix}$} &
    {\small $\begin{bmatrix} 
    0 & 1 & 0 & 0 \\ 
    1 & 0 & 0 & 0 \\
    0 & 0 & 0 & 0 \\
    0 & 0 & 0 & 0 
    \end{bmatrix}$} &
    {\small$\begin{bmatrix} 
    0 & 0 & 1 & 0 \\ 
    0 & 0 & 0 & -1 \\
    1 & 0 & 0 & 0 \\
    0 & -1 & 0 & 0 
    \end{bmatrix}$} & 
    {\small$\begin{bmatrix} 
    0 & 0 & 0 & 1 \\ 
    0 & 0 & 1 & 0 \\
    0 & 1 & 0 & 0 \\
    1 & 0 & 0 & 0 
    \end{bmatrix}$}
      \\ \midrule
    \end{tabular} 
    \caption{Matrices associated with the bilinear form of some 4D algebras.}
    \label{tab:4D_algebras}
\end{table}

In analogy to the numerical examples given in the previous section, this section illustrates Theorem \ref{thm:tau_geral} using the four-dimensional algebras shown in Table \ref{tab:4D_algebras}. Precisely, let $\mathcal{E}=\{\boldsymbol{e}_0,\boldsymbol{e}_1,\boldsymbol{e}_2,\boldsymbol{e}_3\}$ be a basis for a four-dimensional algebra $\mathbb{V}$ and consider the vector-valued function $f_{\mathbb{V}}:\mathbb{V}\to\mathbb{V}$ defined by 
\begin{equation}
    \label{eq:f-example-4D}
    f_{\mathbb{V}}(x) = (\xi_0^2+\xi_1\xi_2+\xi_3^2)\boldsymbol{e}_0+ 
    (\xi_1^2+\xi_2\xi_3+\xi_0^2)\boldsymbol{e}_1+
    (\xi_2^2+\xi_3\xi_0+\xi_1^2)\boldsymbol{e}_2+
    (\xi_3^2+\xi_0\xi_1+\xi_2^2)\boldsymbol{e}_3,
\end{equation}
for all $x \in K$, where 
\begin{equation}
    K = \{x=\xi_0\boldsymbol{e}_0+\xi_1\boldsymbol{e}_1+\xi_2\boldsymbol{e}_2+\xi_3\boldsymbol{e}_3: \xi_0,\xi_1,\xi_2,\xi_3 \in [-1,1]\} \subset \mathbb{V},
\end{equation}
is a compact set. Like the function considered previously, the components of $f_{\mathbb{V}}$ given by \eqref{eq:f-example-4D} are quadratic forms on the components $\xi_0,\xi_1,\xi_2$ and $\xi_3$ of $x$. Therefore, $f_{\mathbb{V}}$ is a continuous but non-linear vector-valued function. {Like the previous example, the function $f_{\mathbb{V}}$ is chosen as it is simple yet poses mild difficulty to be approximated by V-MLP networks, which makes it appropriate for illustrating Theorem \ref{thm:tau_geral}.}

Using the four-dimensional algebras listed in Table \ref{tab:4D_algebras}, we implemented V-MLP networks with 128 hidden neurons with the split \texttt{ReLU} activation function. {For comparison purposes, we also implemented a traditional (real-valued) MLP network with the same architecture}.
All the networks have been trained by minimizing the mean-squared error over a dataset with 1024 samples taken uniformly from $K$ and using the Adam optimizer for 1000 epochs. Figures \ref{fig:Experiment-4D-R} and \ref{fig:Experiment-4D-V} show the MSE by the number of epochs for V-MLP networks with real-valued and vector-valued output weights, respectively. {These figures also show the MSE produced by the traditional MLP network.} 

\begin{figure}
    \centering
    \includegraphics[width=\columnwidth]{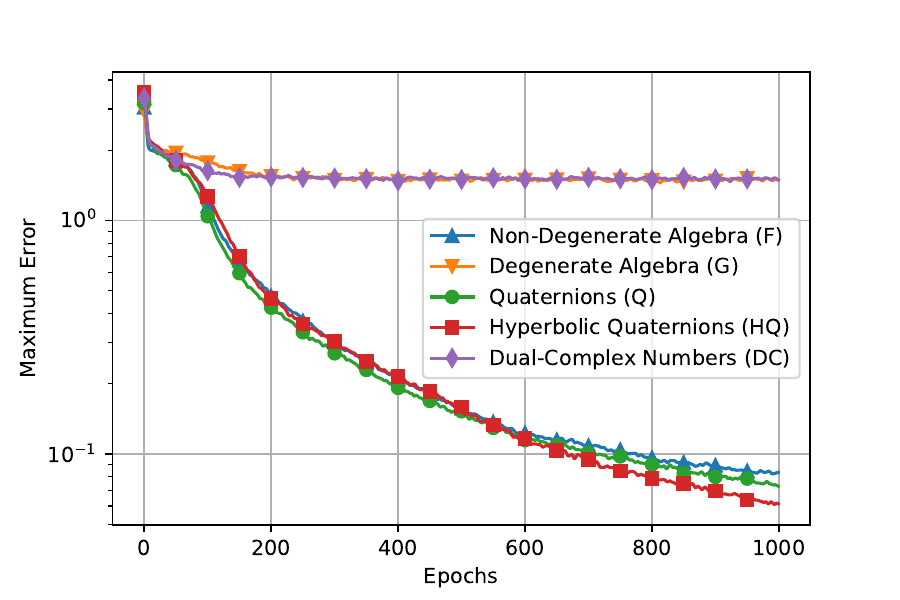}
    \caption{MSE by the number of epochs during the training phase of V-MLP networks with real-valued output weights defined on four-dimensional algebras.}
    \label{fig:Experiment-4D-R}
\end{figure}

Note from Figure \ref{fig:Experiment-4D-R} that the MSE consistently decreased for the V-MLP networks with real-valued output weights that were defined on non-degenerate algebras such as the algebra $\mathbb{F}$, the quaternions $\mathbb{Q}$, and the hyperbolic quaternions $\mathbb{HQ}$. {In contrast, the MSE produced by the networks defined on the degenerate algebras $\mathbb{G}$ and $\mathbb{DC}$ (dual-complex numbers) remained stagnant at $6.19\times 10^{-2}$ and $7.66 \times 10^{-2}$, respectively.}

\begin{figure}
    \centering
    \includegraphics[width=\columnwidth]{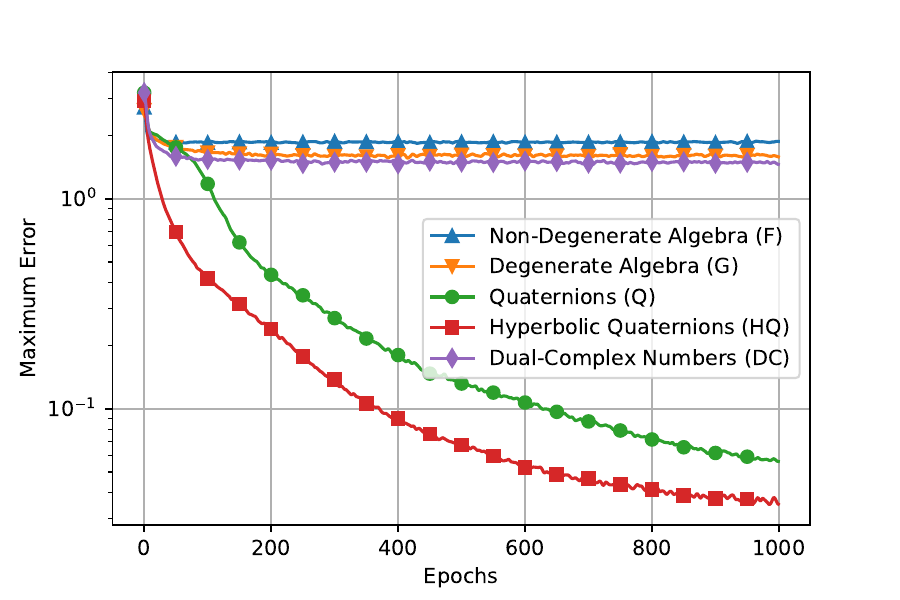}
    \caption{MSE by the number of epochs during the training phase of V-MLP networks with vector-valued output weights defined on four-dimensional algebras.}
    \label{fig:Experiment-4D-V}
\end{figure}

When vector-valued output weights were considered instead of real-valued ones, the MSE consistently decreased for the MLP networks defined on non-degenerate hypercomplex algebras such as $\mathbb{Q}$ and $\mathbb{HQ}$ (see Figure \ref{fig:Experiment-4D-V}). {For the degenerate algebras $\mathbb{G}$ and $\mathbb{DC}$, the MSE remained stagnant at $7.70\times10^{-2}$ and $7.60\times10^{-2}$, respectively. Moreover, notice that the MSE remained stagnant at $1.79\times10^{-1}$ for the V-MLP with vector-valued output weights defined on the non-degenerate algebra $\mathbb{F}$, which is not a hypercomplex algebra.}

In conclusion, the numerical examples provided in this subsection illustrate the following:
\begin{itemize}
    \item V-MLP networks with real-valued output weights defined on non-degenerate algebras exhibit approximation capability.
    \item We cannot ensure the approximation capability for a V-MLP with real-valued output weights defined on degenerate algebras. Indeed, the networks defined on the degenerate algebra $\mathbb{G}$ and dual-complex numbers $\mathbb{DC}$ failed to approximate the continuous vector-valued function $f_{\mathbb{V}}$ given by \eqref{eq:f-example-4D} on the compact set $K \subset \mathbb{V}$.
    \item Hypercomplex-valued MLP networks with hypercomplex-valued output weights also exhibit the approximation capability on non-degenerate algebras.
    \item V-MLP networks with vector-valued output weights may not approximate a vector-valued function if the algebra is not hypercomplex. In fact, the V-MLP networks defined on the non-degenerate algebra $\mathbb{F}$ failed to approximate the continuous function $f_{\mathbb{V}}:\mathbb{V}\to \mathbb{V}$ given by \eqref{eq:f-example-4D} on the compact set $K \subset \mathbb{V}$. 
\end{itemize}

\section{Concluding Remarks}\label{sec:concluding}

The universal approximation theorem states that a single hidden layer multi-layer perceptron (MLP) network can accurately approximate continuous functions with arbitrary precision. This important theoretical result was first proven for real-valued networks in the late 1980s and early 1990s by many researchers under different hypotheses on the activation function \cite{Cybenko1989ApproximationFunction,funahashi1989approximate,hornik1991approximation,leshno1993multilayer,Pinkus1999ApproximationNetworks}. In the subsequent years, the universal approximation theorem was also demonstrated for neural networks based on well-known hypercomplex algebras, such as complex \cite{arena1998neural}, quaternions \cite{ARENA1997335}, and Clifford algebras \cite{Buchholz2001}. However, each of these results was derived independently, leading to a lack of generality in the approximation capability of hypercomplex-valued networks. In this work, we investigate some of the existing universal approximation theorems and link the approximation capability of hypercomplex-valued networks using the broad class of vector-valued neural networks. 

Vector-valued neural networks are similar to real-valued models, but they utilize elements of an algebra $\mathbb{V}$ for weights, bias, inputs, and outputs \cite{valle2024understanding}. An algebra $\mathbb{V}$ is a finite-dimensional vector space with a bilinear operation called multiplication or product (see Definition \ref{def:algebra}). An algebra is considered hypercomplex if the multiplication has a two-sided identity (see Definition \ref{def:hypercomplex-algebra}). The broad class of V-nets includes complex, quaternion, tessarine, and Clifford-valued models \cite{aizenberg11book,arena1998neural,Buchholz2001,Buchholz2008OnPerceptrons,Carniello2021UniversalNetworks,hirose12} and many other hypercomplex-valued neural networks, for which their approximation capability was previously unknown \cite{Grassucci2022PHNNs:Convolutions,popa2016octonion,Wu2020DeepNetworks,Takahashi2021ComparisonControl,vieira2022general}. {Moreover, V-nets are equivalent to the neural networks introduced by Fan et al. using arbitrary bilinear forms \cite{Fan2020BackpropagationProducts}.}

Using the approximation capability of real-valued MLP networks, we showed that a V-MLP network with a split-activation function and real-valued output weights can approximate a continuous vector-valued function to any given precision on a compact, provided the algebra is non-degenerate. We also demonstrated the approximation capability of hypercomplex-valued MLP networks with hypercomplex output weights. These results are summarized in Theorem \ref{thm:tau_geral}, the primary contribution of this paper.

Concluding, the universal approximation theorem provided in this paper serves multiple purposes. Firstly, it consolidates the results regarding the universal approximation capability of V-MLP networks defined on many well-known algebras, eliminating the need to prove this property for each algebra individually. The class of non-degenerate hypercomplex algebras, including the complex and hyperbolic numbers, quaternions, tessarines, and Clifford algebras, are covered in this result. Furthermore, the universal approximation property is now recognized to hold for MLP networks on many other algebras including the Klein group and the octonions \cite{KOBAYASHI2020123,popa2016octonion}. Lastly, Theorem \ref{thm:tau_geral} should promote the use of vector- and hypercomplex-valued networks, which are known to perform well in problems involving multidimensional signals, such as images, video, and 3D movement \cite{parcollet2020survey,vieira2022general}. The universal approximation property strengthens these models' applications, making them better suited than real-valued models for a wider variety of applications.

\appendix

\section{Some Approximation Theorems from the Literature} \label{sec:UATS_diversos}


This appendix begins by reviewing the approximation capability of traditional (real-valued) multilayer perceptron networks. Subsequently, the approximation capability of some hyper-per-complex-valued neural networks is briefly discussed. {We would like to remark that all of these results can be derived from Theorem \ref{thm:tau_geral} proven in this paper.}


\subsection{The Approximation Capability of Traditional Neural Networks} \label{ssec:approximation-traditional-networks}

A traditional multilayer perceptron (MLP) is a feedforward artificial neural network architecture with neurons arranged in layers. Moreover, each neuron in a layer is connected to all neurons in the previous layer; hence, an MLP network is given by a sequence of fully connected or dense layers. The feedforward step through an MLP with a single hidden layer with $M$ neurons can be described by a finite linear combination of the hidden neuron outputs. Formally, the output of a traditional single hidden layer MLP network $\mathcal{N}_{\mathbb{R}}:\mathbb{R}^N \to \mathbb{R}$ is given by
\begin{equation}\label{eq:mlp_real}
\mathcal{N}_{\mathbb{R}}(\boldsymbol{x}) = \sum_{i=1}^{M} \alpha_{i} \psi_{\mathbb{R}}\left(\sum_{j=1}^N w_{ij}x_j  + b_{i} \right),
\end{equation}
for an input $\boldsymbol{x}=(x_1,\ldots,x_N) \in \mathbb{R}^N$. The parameters $w_{ij} \in \mathbb{R}$ and $\alpha_i \in \mathbb{R}$ are the weights between input and hidden layers, and hidden and output layers, respectively, for all $i=1,\ldots,M$ and $j=1,\ldots,N$. Moreover, $\psi_{\mathbb{R}}:\mathbb{R}\to \mathbb{R}$ is the activation function, and $b_{i} \in \mathbb{R}$ represents the bias term for the $i$th neuron in the hidden layer, for $i=1,\ldots,M$. 

\begin{remark} \label{rmk:MLP_linear}
Because $\mathcal{L}_i(\boldsymbol{x}) = \sum_{j=1}^N w_{ij}x_j$ defines a linear functional $\mathcal{L}_i:\mathbb{R}^N \to \mathbb{R}$, a single hidden layer MLP network can alternatively be written as
\begin{equation} \label{eq:R-MLP-ell}
\mathcal{N}_\mathbb{R}(\boldsymbol{x}) = \sum_{i=1}^{M} \alpha_{i} \psi_{\mathbb{R}}( \mathcal{L}_i(\boldsymbol{x}) + b_{i}), \quad \forall \boldsymbol{x} \in \mathbb{R}^n,
\end{equation}
where $\mathcal{L}_1,\mathcal{L}_2,\ldots,\mathcal{L}_M:\mathbb{R}^N \to \mathbb{R}$ are linear functionals.
\end{remark}

The class of all single hidden layer MLP networks with activation function $\psi_{\mathbb{R}}:\mathbb{R}\to \mathbb{R}$ is denoted by
\begin{equation}\label{eq:class_MLP_real}
    \mathcal{H}_{\mathbb{R}} = \left\{ \mathcal{N}_{\mathbb{R}}(\boldsymbol{x}) = \sum_{i=1}^{M} \alpha_{i} \psi_{\mathbb{R}}(\mathcal{L}_i(\boldsymbol{x}) + b_{i}): M \in \mathbb{N} \right\}.
\end{equation}
The universal approximation theorem provides conditions on the activation function $\psi_{\mathbb{R}}$ under which the class $\mathcal{H}_{\mathbb{R}}$ is dense in the set of all continuous functions defined on a compact set $K \subset \mathbb{R}^N$. Examples of activation function $\psi_{\mathbb{R}}:\mathbb{R} \to \mathbb{R}$ which ensures the approximation capability of single hidden layer MLP networks include sigmoid functions and the modern rectified linear units. The logistic function, defined by  
\begin{equation}
\label{eq:sigmoid}
    \sigma(x)=\frac{1}{1+e^{-x}}, \quad \forall x \in \mathbb{R},
\end{equation}
is an instance of a sigmoid activation function. The rectified linear unit activation function is defined as follows for all $x \in \mathbb{R}$:
\begin{equation} \label{eq:relu}
\mathtt{ReLU}(x)= \begin{cases}x, & x \geq 0, \\ 0, & \text{otherwise}.\end{cases} 
\end{equation}

As pointed out in the introduction, conditions on the activation function that yield the approximation capability of single hidden layer MLP networks with sufficient hidden neurons have been given by many prominent researchers in the late 1980s and early 1990s. Namely, Cybenko proved the approximation capability considering the broad class of so-called discriminatory continuous functions, which include sigmoid activation functions \cite{Cybenko1989ApproximationFunction}. Similarly, Funahashi proved the universal approximation theorem with nonconstant, bounded, and monotone continuous activation function \cite{funahashi1989approximate}. Independently, Hornik et al. derived the approximation capability of single hidden layer MLP networks analogous to Funahashi but without assuming the continuity of the activation function \cite{Hornik1989MultilayerApproximators}. {Also, Chen et al. provided approximation results considering bounded sigmoid activation functions, not necessarily continuous \cite{chen1995universal,chen1990constructive}.} The reader interested in a historical account of the main results on the approximation capability of MLP networks is invited to consult  \cite{Ismailov2021RidgeNetworks,Pinkus1999ApproximationNetworks}. Accordingly, based on Leshno et al. \cite{leshno1993multilayer}, the universal approximation theorem for a single hidden layer neural network with continuous activation functions can be formulated as follows:

\begin{theorem}[Universal Approximation Theorem \cite{Pinkus1999ApproximationNetworks}] \label{thm:Real-TAU}
Let $\psi_{\mathbb{R}}:\mathbb{R} \to \mathbb{R}$ be a continuous function.
The class of all real-valued neural networks $\mathcal{H}_{\mathbb{R}}$ defined by \eqref{eq:class_MLP_real} is dense in the class $\mathcal{C}(K)$ of all real-valued continuous functions on a compact $K \subseteq \mathbb{R}^N$ if and only if $\psi_{\mathbb{R}}$ is not a polynomial.
\end{theorem}

\begin{remark}
Note that Theorem \ref{thm:Real-TAU} implies the following: Given $f_\mathbb{R}:K \to \mathbb{R}$ and $\varepsilon>0$, there exist a single hidden layer MLP network given by \eqref{eq:mlp_real} such that
\begin{equation}
    | f_\mathbb{R}(\boldsymbol{x})-\mathcal{N}_\mathbb{R}(\boldsymbol{x})|< \varepsilon, \quad \forall \boldsymbol{x} \in K,
\end{equation}
if $\psi_{\mathbb{R}}:\mathbb{R} \to \mathbb{R}$ is a continuous non-polynomial activation function. In particular, the logistic and the $\mathtt{ReLU}$ functions, defined respectively by \eqref{eq:sigmoid} and \eqref{eq:relu}, yield the universal approximation capability for single hidden layer MLP networks. 
\end{remark}

{
We would like to remark that although Theorem \ref{thm:Real-TAU} dates to the 1990s, the approximation capability of neural networks is yet a field of active research. Some recent works address, for example, error bounds for neural networks \cite{Almira2021NegativeNetworks,guliyev2018approximation, Ismailov2024ApproximationWeights} and the approximation capability of networks besides the MLP architecture \cite{Ismayilova2024OnNetworks,NEURIPS2018_03bfc1d4}.
}

\subsection{Complex-valued MLP Networks}\label{subsec:complex}

The structure of a complex-valued MLP ($\mathbb{C}$MLP) is equivalent to that of a real-valued MLP, except that input and output signals, weights, and bias are complex numbers instead of real values. Additionally, the activation functions are complex-valued functions \cite{arena1998neural}. 

Note that the logistic function given by \eqref{eq:sigmoid} can be generalized to complex parameters using Euler's formula  $ e^{t \ii} = \cos (t) + \ii \sin(t)$ as follows for all $x \in \mathbb{C}$:
\begin{eqnarray} \label{eq:sigmoid_C}
\sigma_{\mathbb{C}}(x)=\frac{1}{1+e^{-x}}.
\end{eqnarray}
However, Arena et al. (1998) noted that the universal approximation property is generally not applicable when using single hidden layer $\mathbb{C}$MLP networks with the activation function given by \eqref{eq:sigmoid_C} \cite{arena1998neural}. Nonetheless, based on the works of Cybenko, they proved the universal approximation theorem for $\mathbb{C}$MLP networks with split continuous discriminatory activation functions \cite{arena1998neural}. A complex-valued split activation function $\psi_{\mathbb{C}}:\mathbb{C} \to \mathbb{C}$ is defined as follows using a real-valued function $\psi_{\mathbb{R}}:\mathbb{R} \to \mathbb{R}$ 
\begin{eqnarray}\label{eq:split_complex}
\psi_{\mathbb{C}}(x) = \psi_{\mathbb{R}}(\xi_0)+ \ii \psi_{\mathbb{R}}(\xi_1), \quad \forall x=\xi_{0}+ \ii \xi_{1} \in \mathbb{C}.
\end{eqnarray} 
We would like to point out that the universal approximation capability of complex-valued neural networks has been recently further investigated by Voigtlaender, where the author characterizes the requirements on the activation function \cite{Voigtlaender2023TheNetworks}.

\subsection{Quaternion-valued MLP Networks}\label{subsec:quat}

In the same vein, Arena et al. also defined quaternion-valued MLP ($\mathbb{Q}$MLP) by replacing the real input and output, weights and biases, with quaternion numbers. Using the result of Cybenko, they proceeded to prove that single hidden layer $\mathbb{Q}$MLPs with real output weights and split continuous discriminatory activation functions are universal approximators in the set of continuous quaternion-valued functions \cite{ARENA1997335}. A split quaternion function $\psi_{\mathbb{Q}}:\mathbb{Q} \to \mathbb{Q}$ is defined as follows using a real-valued function $\psi_{\mathbb{R}}:\mathbb{R}\to \mathbb{R}$:
\begin{equation}
    \psi_{\mathbb{Q}}(x) =  \psi_{\mathbb{R}}(\xi_0)+ \psi_{\mathbb{R}}(\xi_1) \ii_1 +   \psi_{\mathbb{R}}{(\xi_2)} \ii_2+ \psi_{\mathbb{R}}{(\xi_3)}\ii_3,
\end{equation}
for $x = \xi_0 + \xi_1\ii_1 + \xi_2\ii_2 + \xi_3\ii_3 \in \mathbb{Q}$.

\subsection{Hyperbolic-valued MLP Networks}\label{subsec:hyperbolic}

The hyperbolic numbers, denoted by $\mathbb{U}$, constitute a two-dimensional hypercomplex algebra similar to complex numbers with an imaginary unit $\ii$ such that $\ii^2=1$ \cite{buchholz2000hyperbolic,nitta08,nitta18}. Therefore, the multiplication of two hyperbolic numbers $x=\xi_0+\xi_1\ii$ and $y=\eta_0+\eta_1\ii$ yields $xy = \mathcal{B}_0(x,y) + \mathcal{B}_1(x,y)\ii$, where $\mathcal{B}_0(x,y) = \xi_0\eta_0 + \xi_1\eta_1$ and $\mathcal{B}_1(x,y) = \xi_0\eta_1 + \xi_1 \eta_0$.

In the year 2000, Buchholz and Sommer introduced an MLP based on hyperbolic numbers, the aptly named hyperbolic multilayer perceptron ($\mathbb{U}$MLP). This network with a split logistic activation function defined analogously to \eqref{eq:split_complex}, is also a universal approximator \cite{buchholz2000hyperbolic}. Buchholz and Sommer provided experiments highlighting that the $\mathbb{U}$MLP can learn tasks with underlying hyperbolic properties much more accurately and efficiently than $\mathbb{C}$MLP and real-valued MLP networks.

\subsection{Tessarine-valued MLP Networks}\label{subsec:tessarine}

Tessarines, usually denoted by $\mathbb{T}$, are a four-dimensional commutative hypercomplex algebra similar to quaternions \cite{CERRONI2017232}. Like the quaternions, tessarines have been used for digital signal processing
\cite{Ortolani2017OnProcessing,Navarro-Moreno2020TessarineCondition}.
Recently, Carniello et al. experimented with networks with inputs, outputs, and parameters in the tessarine algebra \cite{Carniello2021UniversalNetworks}.
The researchers proposed the $\mathbb{T}$MLP, a multilayer perceptron architecture similar to the complex, quaternion, and hyperbolic networks mentioned above but based on tessarines. The authors then proceeded to show that the proposed $\mathbb{T}$MLP is a universal approximator for continuous functions defined on a compact set. Experiments show that the tessarine-valued network is a powerful approximator, presenting superior performance when compared to the real-valued MLP in a task of approximating tessarine-valued functions \cite{Carniello2021UniversalNetworks,Senna2021TessarineClassification}.

\subsection{Clifford-valued MLP Networks}\label{subsec:clifford}

Buchholz and Sommer investigated in 2001 neural networks that use Clifford algebras \cite{Buchholz2001}. They proved that the universal approximation property applies to multilayer perceptrons (MLPs) networks that are based on non-degenerate Clifford algebras. They also pointed out that degenerate Clifford algebras may result in models that lack universal approximation capabilities.

It is worth noting that Buchholz and Sommer considered sigmoid activation functions. However, it is possible to show that Clifford MLPs are universal approximators with the split $\mathtt{ReLU}$ activation function as well.

\subsection{Octonion-valued MLP Networks}

Octonions, developed independently by Graves and Cayley, constitute an eight-dimensional hypercomplex algebra that extends the complex numbers and quaternion \cite{baez02}. Octonions have important applications in fields like string theory, special relativity, and quantum logic, although they may not be as well-known as quaternions and complex numbers \cite{baez02}.

An octonion-valued neural network is a type of machine learning model used to process octonion-valued inputs and outputs, and whose trainable parameters are also octonions \cite{kuroe16,popa2016octonion,castro17cnmac,Saoud2020MetacognitiveAnalysis,Wu2020DeepNetworks}. In particular, Popa introduced a class of fully connected neural networks based on octonion algebra \cite{popa2016octonion}. This class can be viewed as a generalization of complex and quaternion-valued neural networks but does not fall into the Clifford-valued neural networks category due to the octonion product's non-associativity.  The approximation capability for this model has not yet been proven in the literature, but it is covered by Theorem \ref{thm:tau_geral}.

\section*{Funding information}

M.E. Valle is supported by the National Council for Scientific and Technological Development (CNPq) under grant no. 315820/2021-7 and the São Paulo Research Foundation (FAPESP) under grant no. 2022/01831-2. Wington is supported by the Instituto de Pesquisa Eldorado, Campinas--Brazil, and the Coordenação de Aperfeiçoamento de Pessoal de Nível Superior (CAPES) -- Finance Code 001.

\section*{Declaration of Competing Interest}

The authors declare they have no conflicts of interest.

\section*{Credit authorship contribution statement}

\textbf{Conceptualization}: M.E. Valle and W.L. Vital; 
\textbf{Formal analysis}: M.E. Valle and W.L. Vital;
\textbf{Funding acquisition}: M.E. Valle and W.L. Vital;
\textbf{Investigation}: M.E. Valle and W.L. Vital;
\textbf{Methodology}: G. Vieira, M.E. Valle, and W.L. Vital;
\textbf{Software}: M.E. Valle and W.L. Vital;
\textbf{Supervision}: M.E. Valle;
\textbf{Validation}: G. Vieira
\textbf{Visualization}: M.E. Valle and W.L. Vital;
\textbf{Writing}: G. Vieira, M.E. Valle and W.L. Vital;

\section*{Declaration of generative AI and AI-assisted technologies in the writing process}

During the preparation of this work, the authors used Grammarly solely to improve language and readability. After using this tool, the authors reviewed and edited the content as needed and take full responsibility for the content of the publication.


\begin{thebibliography}{10}

\bibitem{aizenberg11book}
{\sc Aizenberg, I.~N.}
\newblock {\em {Complex-Valued Neural Networks with Multi-Valued Neurons}},
  1~ed., vol.~353 of {\em Studies in Computational Intelligence}.
\newblock Springer, Berlin Heidelberg, 2011.

\bibitem{Almira2021NegativeNetworks}
{\sc Almira, J., Lopez-de Teruel, P., Romero-L{\'{o}}pez, D., and Voigtlaender,
  F.}
\newblock {Negative results for approximation using single layer and multilayer
  feedforward neural networks}.
\newblock {\em Journal of Mathematical Analysis and Applications 494}, 1 (2
  2021), 124584.

\bibitem{ARENA1997335}
{\sc Arena, P., Fortuna, L., Muscato, G., and Xibilia, M.}
\newblock Multilayer perceptrons to approximate quaternion valued functions.
\newblock {\em Neural Networks 10}, 2 (1 1997), 335--342.

\bibitem{arena1998neural}
{\sc Arena, P., Fortuna, L., Muscato, G., and Xibilia, M.~G.}
\newblock {\em Neural networks in multidimensional domains: fundamentals and
  new trends in modeling and control}.
\newblock Springer London, 1998.

\bibitem{baez02}
{\sc Baez, J.~C.}
\newblock {The octonions}.
\newblock {\em Bulletin of the American Mathematical Society 39\/} (2002),
  145--205.

\bibitem{Brignone2022EfficientDomain}
{\sc Brignone, C., Mancini, G., Grassucci, E., Uncini, A., and Comminiello, D.}
\newblock {Efficient Sound Event Localization and Detection in the Quaternion
  Domain}.
\newblock {\em IEEE Transactions on Circuits and Systems II: Express Briefs\/}
  (2022), 1--5.

\bibitem{buchholz2000hyperbolic}
{\sc Buchholz, S., and Sommer, G.}
\newblock A hyperbolic multilayer perceptron.
\newblock In {\em Proceedings of the IEEE-INNS-ENNS International Joint
  Conference on Neural Networks (IJCNN 2000)\/} (jul 2000), vol.~2, IEEE,
  pp.~129--133.

\bibitem{Buchholz2001}
{\sc Buchholz, S., and Sommer, G.}
\newblock {\em Clifford Algebra Multilayer Perceptrons}.
\newblock Springer Berlin Heidelberg, Berlin, Heidelberg, 2001, pp.~315--334.

\bibitem{Buchholz2008OnPerceptrons}
{\sc Buchholz, S., and Sommer, G.}
\newblock {On Clifford neurons and Clifford multi-layer perceptrons}.
\newblock {\em Neural Networks 21}, 7 (9 2008), 925--935.

\bibitem{Carniello2021UniversalNetworks}
{\sc Carniello, R. A.~F., Vital, W.~L., and Valle, M.~E.}
\newblock {Universal Approximation Theorem for Tessarine-Valued Neural
  Networks}.
\newblock {\em Anais do Encontro Nacional de Intelig{\^{e}}ncia Artificial e
  Computacional (ENIAC)\/} (11 2021), 233--243.

\bibitem{Catoni2008TheSpace-Time}
{\sc Catoni, F., Boccaletti, D., Cannata, R., Catoni, V., Nichelatti, E., and
  Zampetti, P.}
\newblock {\em {The Mathematics of Minkowski Space-Time}}.
\newblock Birkh{\"{a}}user Basel, 2008.

\bibitem{CERRONI2017232}
{\sc Cerroni, C.}
\newblock From the theory of congeneric surd equations to segre's bicomplex
  numbers.
\newblock {\em Historia Mathematica 44}, 3 (2017), 232--251.

\bibitem{chen1995universal}
{\sc Chen, T., and Chen, H.}
\newblock Universal approximation to nonlinear operators by neural networks
  with arbitrary activation functions and its application to dynamical systems.
\newblock {\em IEEE Transactions on Neural Networks 6}, 4 (1995), 911--917.

\bibitem{chen1990constructive}
{\sc Chen, T., Chen, H., and Liu, R.-w.}
\newblock A constructive proof and an extension of {C}ybenko's approximation
  theorem.
\newblock In {\em Computing Science and Statistics\/} (New York, NY, 1992),
  C.~Page and R.~LePage, Eds., Springer New York, pp.~163--168.

\bibitem{comminiello2024demystifying}
{\sc Comminiello, D., Grassucci, E., Mandic, D.~P., and Uncini, A.}
\newblock Demystifying the hypercomplex: Inductive biases in hypercomplex deep
  learning.
\newblock {\em IEEE Signal Processing Magazine\/} (2024).
\newblock Accepted for publication.

\bibitem{Cybenko1989ApproximationFunction}
{\sc Cybenko, G.}
\newblock {Approximation by superpositions of a sigmoidal function}.
\newblock {\em Mathematics of Control, Signals and Systems 1989 2:4 2}, 4 (12
  1989), 303--314.

\bibitem{castro17cnmac}
{\sc De~Castro, F.~Z., and Valle, M.~E.}
\newblock {Continuous-Valued Octonionic Hopfield Neural Network}.
\newblock In {\em Proceedings Series of the Brazilian Society of Computational
  and Applied Mathematics\/} (S{\~{a}}o Jos{\'{e}} dos Campos -- Brazil, 2
  2018), vol.~6, pp.~1--7.

\bibitem{Ding2020OnlinePrediction}
{\sc Ding, T., and Hirose, A.}
\newblock {Online Regularization of Complex-Valued Neural Networks for
  Structure Optimization in Wireless-Communication Channel Prediction}.
\newblock {\em IEEE Access 8\/} (2020), 143706--143722.

\bibitem{Fan2020BackpropagationProducts}
{\sc Fan, Z.-C., Chan, T.-S.~T., Yang, Y.-H., and Jang, J.-S.~R.}
\newblock {Backpropagation With N-D Vector-Valued Neurons Using Arbitrary
  Bilinear Products}.
\newblock {\em IEEE Transactions on Neural Networks and Learning Systems\/}
  (2020), 1--15.

\bibitem{funahashi1989approximate}
{\sc Funahashi, K.-I.}
\newblock On the approximate realization of continuous mappings by neural
  networks.
\newblock {\em Neural networks 2}, 3 (1989), 183--192.

\bibitem{Grassucci2023DualRepresentation}
{\sc Grassucci, E., Mancini, G., Brignone, C., Uncini, A., and Comminiello, D.}
\newblock {Dual quaternion ambisonics array for six-degree-of-freedom acoustic
  representation}.
\newblock {\em Pattern Recognition Letters 166\/} (2 2023), 24--30.

\bibitem{Grassucci2022PHNNs:Convolutions}
{\sc Grassucci, E., Zhang, A., and Comminiello, D.}
\newblock {PHNNs: Lightweight Neural Networks via Parameterized Hypercomplex
  Convolutions}.
\newblock {\em IEEE Transactions on Neural Networks and Learning Systems\/} (10
  2022), 1--13.

\bibitem{Guizzo2023LearningDomain}
{\sc Guizzo, E., Weyde, T., Scardapane, S., and Comminiello, D.}
\newblock {Learning Speech Emotion Representations in the Quaternion Domain}.
\newblock {\em IEEE/ACM Transactions on Audio, Speech, and Language Processing
  31\/} (2023), 1200--1212.

\bibitem{guliyev2018approximation}
{\sc Guliyev, N.~J., and Ismailov, V.~E.}
\newblock On the approximation by single hidden layer feedforward neural
  networks with fixed weights.
\newblock {\em Neural Networks 98\/} (2018), 296--304.

\bibitem{hirose12}
{\sc Hirose, A.}
\newblock {\em {Complex-Valued Neural Networks}}, 2nd edition~ed.
\newblock Studies in Computational Intelligence. Springer, Heidelberg, Germany,
  2012.

\bibitem{hornik1991approximation}
{\sc Hornik, K.}
\newblock Approximation capabilities of multilayer feedforward networks.
\newblock {\em Neural networks 4}, 2 (1991), 251--257.

\bibitem{Hornik1989MultilayerApproximators}
{\sc Hornik, K., Stinchcombe, M., and White, H.}
\newblock {Multilayer feedforward networks are universal approximators}.
\newblock {\em Neural Networks 2}, 5 (1 1989), 359--366.

\bibitem{huang2013klein}
{\sc Huang, J.-S., and Yu, J.}
\newblock Klein four-subgroups of lie algebra automorphisms.
\newblock {\em Pacific Journal of Mathematics 262}, 2 (2013), 397--420.

\bibitem{Ismailov2021RidgeNetworks}
{\sc Ismailov, V.}
\newblock {\em {Ridge Functions and Applications in Neural Networks}}, 1st~ed.,
  vol.~263.
\newblock American Mathematical Society, Providence, Rhode Island, 2021.

\bibitem{Ismailov2024ApproximationWeights}
{\sc Ismailov, V.~E.}
\newblock {Approximation error of single hidden layer neural networks with
  fixed weights}.
\newblock {\em Information Processing Letters 185\/} (3 2024), 106467.

\bibitem{Ismayilova2024OnNetworks}
{\sc Ismayilova, A., and Ismailov, V.~E.}
\newblock {On the Kolmogorov neural networks}.
\newblock {\em Neural Networks 176\/} (8 2024), 106333.

\bibitem{IturrinoGarcia2023PowerLines}
{\sc Iturrino~Garcia, C.~A., Bindi, M., Corti, F., Luchetta, A., Grasso, F.,
  Paolucci, L., Piccirilli, M.~C., and Aizenberg, I.}
\newblock {Power Quality Analysis Based on Machine Learning Methods for
  Low-Voltage Electrical Distribution Lines}.
\newblock {\em Energies 16}, 9 (4 2023), 3627.

\bibitem{kantor1989hypercomplex}
{\sc Kantor, I., and Solodovnikov, A.}
\newblock {\em Hypercomplex numbers: an elementary introduction to algebras},
  vol.~302.
\newblock Vol. 302. New York: Springer-Verlag,, 1989.

\bibitem{KOBAYASHI2020123}
{\sc Kobayashi, M.}
\newblock Hopfield neural networks using klein four-group.
\newblock {\em Neurocomputing 387\/} (2020), 123--128.

\bibitem{Korevaar2004TauberianTheory}
{\sc Korevaar, J.}
\newblock {\em {Tauberian Theory}}, vol.~329.
\newblock Springer Berlin Heidelberg, Berlin, Heidelberg, 2004.

\bibitem{kuroe16}
{\sc Kuroe, Y., and Iima, H.}
\newblock {A model of Hopfield-type octonion neural networks and existing
  conditions of energy functions}.
\newblock In {\em 2016 International Joint Conference on Neural Networks
  (IJCNN)\/} (2016), pp.~4426--4430.

\bibitem{leshno1993multilayer}
{\sc Leshno, M., Lin, V.~Y., Pinkus, A., and Schocken, S.}
\newblock Multilayer feedforward networks with a nonpolynomial activation
  function can approximate any function.
\newblock {\em Neural networks 6}, 6 (1993), 861--867.

\bibitem{NEURIPS2018_03bfc1d4}
{\sc Lin, H., and Jegelka, S.}
\newblock Resnet with one-neuron hidden layers is a universal approximator.
\newblock In {\em Advances in Neural Information Processing Systems\/} (2018),
  S.~Bengio, H.~Wallach, H.~Larochelle, K.~Grauman, N.~Cesa-Bianchi, and
  R.~Garnett, Eds., vol.~31, Curran Associates, Inc.

\bibitem{Matsumoto2022Full-LearningClassification}
{\sc Matsumoto, Y., Natsuaki, R., and Hirose, A.}
\newblock {Full-Learning Rotational Quaternion Convolutional Neural Networks
  and Confluence of Differently Represented Data for PolSAR Land
  Classification}.
\newblock {\em IEEE Journal of Selected Topics in Applied Earth Observations
  and Remote Sensing 15\/} (2022), 2914--2928.

\bibitem{mcculloch1943logical}
{\sc McCulloch, W.~S., and Pitts, W.}
\newblock A logical calculus of the ideas immanent in nervous activity.
\newblock {\em The bulletin of mathematical biophysics 5\/} (1943), 115--133.

\bibitem{Navarro-Moreno2020TessarineCondition}
{\sc Navarro-Moreno, J., Fern{\'{a}}ndez-Alcal{\'{a}}, R.~M.,
  Jim{\'{e}}nez-L{\'{o}}pez, J.~D., and Ruiz-Molina, J.~C.}
\newblock {Tessarine Signal Processing under the T-Properness Condition}.
\newblock {\em Journal of the Franklin Institute\/} (8 2020).

\bibitem{Navarro-Moreno2022ProperDomain}
{\sc Navarro-Moreno, J., Fern{\'{a}}ndez-Alcal{\'{a}}, R.~M., and Ruiz-Molina,
  J.~C.}
\newblock {Proper ARMA Modeling and Forecasting in the Generalized
  Segre{\&}rsquo;s Quaternions Domain}.
\newblock {\em Mathematics 2022, Vol. 10, Page 1083 10}, 7 (3 2022), 1083.

\bibitem{nitta08}
{\sc Nitta, T., and Buchholz, S.}
\newblock {On the decision boundaries of hyperbolic neurons}.
\newblock In {\em 2008 IEEE International Joint Conference on Neural Networks
  (IEEE World Congress on Computational Intelligence)\/} (2008),
  pp.~2974--2980.

\bibitem{nitta18}
{\sc Nitta, T., and Kuroe, Y.}
\newblock {Hyperbolic gradient operator and hyperbolic back-propagation
  learning algorithms}.
\newblock {\em IEEE Transactions on Neural Networks and Learning Systems 29}, 5
  (2018), 1689--1702.

\bibitem{Ortolani2017OnProcessing}
{\sc Ortolani, F., Comminiello, D., Scarpiniti, M., and Uncini, A.}
\newblock {On 4-Dimensional Hypercomplex Algebras in Adaptive Signal
  Processing}.
\newblock {\em Smart Innovation, Systems and Technologies 102\/} (6 2017),
  131--140.

\bibitem{parcollet2020survey}
{\sc Parcollet, T., Morchid, M., and Linar{\`e}s, G.}
\newblock A survey of quaternion neural networks.
\newblock {\em Artificial Intelligence Review 53}, 4 (4 2020), 2957--2982.

\bibitem{Pinkus1999ApproximationNetworks}
{\sc Pinkus, A.}
\newblock {Approximation theory of the MLP model in neural networks}.
\newblock {\em Acta Numerica 8\/} (1 1999), 143--195.

\bibitem{popa2016octonion}
{\sc Popa, C.-A.}
\newblock Octonion-valued neural networks.
\newblock In {\em International Conference on Artificial Neural Networks\/}
  (2016), Springer, pp.~435--443.

\bibitem{rosenblatt1958perceptron}
{\sc Rosenblatt, F.}
\newblock {The perceptron: a probabilistic model for information storage and
  organization in the brain.}
\newblock {\em Psychological review 65}, 6 (12 1958), 386--408.

\bibitem{Saoud2020MetacognitiveAnalysis}
{\sc Saoud, L.~S., and Ghorbani, R.}
\newblock {Metacognitive Octonion-Valued Neural Networks as They Relate to Time
  Series Analysis}.
\newblock {\em IEEE Transactions on Neural Networks and Learning Systems 31}, 2
  (2 2020), 539--548.

\bibitem{Schafer}
{\sc Schafer, R.}
\newblock On the algebras formed by the {C}ayley-{D}ickson process.
\newblock {\em American Journal of Mathematics 76}, 2 (1954), 435--46.

\bibitem{Schafer1961AnAlgebras}
{\sc Schafer, R.}
\newblock {\em {An Introduction to Nonassociative Algebras}}.
\newblock Project Gutenberg, 1961.

\bibitem{Senna2021TessarineClassification}
{\sc Senna, F. R.~d., and Valle, M.~E.}
\newblock {Tessarine and Quaternion-Valued Deep Neural Networks for Image
  Classification}.
\newblock {\em Anais do Encontro Nacional de Intelig{\^{e}}ncia Artificial e
  Computacional (ENIAC)\/} (11 2021), 350--361.

\bibitem{shang14}
{\sc Shang, F., and Hirose, A.}
\newblock {Quaternion Neural-Network-Based PolSAR Land Classification in
  Poincare-Sphere-Parameter Space}.
\newblock {\em IEEE Transactions on Geoscience and Remote Sensing 52\/} (2014),
  5693--5703.

\bibitem{Takahashi2021ComparisonControl}
{\sc Takahashi, K.}
\newblock {Comparison of high-dimensional neural networks using hypercomplex
  numbers in a robot manipulator control}.
\newblock {\em Artificial Life and Robotics 26}, 3 (8 2021), 367--377.

\bibitem{valle2024understanding}
{\sc Valle, M.~E.}
\newblock Understanding vector-valued neural networks and their relationship
  with real and hypercomplex-valued neural networks.
\newblock {\em IEEE Signal Processing Magazine\/} (2024).
\newblock Accepted for publication.

\bibitem{vaz16}
{\sc Vaz, J., and da~Rocha, R.}
\newblock {\em {An Introduction to Clifford Algebras and Spinors}}.
\newblock Oxford University Press, 2016.

\bibitem{Vieira2022AcuteNetworks}
{\sc Vieira, G., and Valle, M.~E.}
\newblock {Acute Lymphoblastic Leukemia Detection Using Hypercomplex-Valued
  Convolutional Neural Networks}.
\newblock In {\em 2022 International Joint Conference on Neural Networks
  (IJCNN)\/} (7 2022), IEEE, pp.~1--8.

\bibitem{vieira2022general}
{\sc Vieira, G., and Valle, M.~E.}
\newblock A general framework for hypercomplex-valued extreme learning
  machines.
\newblock {\em Journal of Computational Mathematics and Data Science 3\/}
  (2022), 100032.

\bibitem{Vital2022ExtendingNetworks}
{\sc Vital, W.~L., Vieira, G., and Valle, M.~E.}
\newblock {Extending the Universal Approximation Theorem for a Broad Class of
  Hypercomplex-Valued Neural Networks}.
\newblock {\em Lecture Notes in Computer Science 13654 LNAI\/} (2022),
  646--660.

\bibitem{Voigtlaender2023TheNetworks}
{\sc Voigtlaender, F.}
\newblock {The universal approximation theorem for complex-valued neural
  networks}.
\newblock {\em Applied and Computational Harmonic Analysis 64\/} (5 2023),
  33--61.

\bibitem{Wiener1932TauberianTheorems}
{\sc Wiener, N.}
\newblock {Tauberian Theorems}.
\newblock {\em The Annals of Mathematics 33}, 1 (1 1932), 1.

\bibitem{Wu2020DeepNetworks}
{\sc Wu, J., Xu, L., Wu, F., Kong, Y., Senhadji, L., and Shu, H.}
\newblock {Deep octonion networks}.
\newblock {\em Neurocomputing 397\/} (7 2020), 179--191.

\end{thebibliography}

\end{document}